\documentclass[acmsmall]{acmart}

\AtBeginDocument{%
  \providecommand\BibTeX{{%
    \normalfont B\kern-0.5em{\scshape i\kern-0.25em b}\kern-0.8em\TeX}}}

\setcopyright{rightsretained}
\acmJournal{POMACS}
\acmYear{2021} \acmVolume{5} \acmNumber{1}
\acmArticle{2} \acmMonth{3} \acmPrice{} 
\acmDOI{10.1145/3447380}

\received{October 2020}
\received[revised]{December 2020}
\received[accepted]{January 2021}


\usepackage{amsmath,amsthm,amscd,amsfonts,dsfont,bm} 
\usepackage{graphicx,fancyhdr,color}
\usepackage[ruled]{algorithm2e}
\usepackage{ifthen}
\usepackage{algpseudocode}
\newtheorem{Lemma}{Lemma}
\newtheorem{Remark}{Remark}
\newtheorem{Theorem}{Theorem}
\newtheorem{Definition}{Definition}
\newtheorem{Property}{Property}

\usepackage[T1]{fontenc}

\usepackage{aecompl}

\usepackage{caption}
\usepackage{subcaption}

\usepackage{eqparbox} 

\ifodd 1
\newcommand{\com}[1]{\textbf{\color{red}(YL : #1)}} 
\newcommand{\clar}[1]{\textbf{\color{green}(NEED CLARIFICATION: #1)}}
\newcommand{\response}[1]{\textbf{\color{magenta}(RESPONSE: #1)}} 
\else

\newcommand{\com}[1]{}
\newcommand{\clar}[1]{}
\newcommand{\response}[1]{}
\fi

\newtheorem{innercustomlem}{Lemma}
\newenvironment{customlem}[1]
  {\renewcommand\theinnercustomlem{#1}\innercustomlem}
  {\endinnercustomlem}

\def\rep#1{(\ref{#1})}
\newcommand{\1}{\mathbf{1}}

\usepackage{hyperref}

\usepackage{verbatim}

\newlength\tindent
\renewcommand{\indent}{\hspace*{\tindent}}

\newcommand{\algcom}[1]{\textsl{\color{blue}{\footnotesize #1}}}

\def\scr#1{{\mathcal #1}}
\newcommand{\R}{\mathbb{R}}
\newcommand{\bbb}{\mathbb}
\begin{document}

\title{Federated Bandit: A Gossiping Approach}

\author{Zhaowei Zhu}
\authornote{Both authors contributed equally to this research.}
\affiliation{%
  \institution{University of California, Santa Cruz}
  \streetaddress{1156 High St}
  \city{Santa Cruz}
  \state{CA}
  \country{USA}
  \postcode{95064}
}
\email{zwzhu@ucsc.edu}
\author{Jingxuan Zhu}
\authornotemark[1]
\email{jingxuan.zhu@stonybrook.edu}
\affiliation{%
  \institution{Stony Brook University}
  \streetaddress{100 Nicolls Rd}
  \city{Stony Brook}
  \state{NY}
  \country{USA}
  \postcode{11794}
}
\author{Ji Liu}
\email{ji.liu@stonybrook.edu}
\affiliation{%
  \institution{Stony Brook University}
  \streetaddress{100 Nicolls Rd}
  \city{Stony Brook}
  \state{NY}
  \country{USA}
  \postcode{11794}
}
\author{Yang Liu}
\email{yangliu@ucsc.edu}
\affiliation{%
  \institution{University of California, Santa Cruz}
  \streetaddress{1156 High St}
  \city{Santa Cruz}
  \state{CA}
  \country{USA}
  \postcode{95064}
}

\begin{abstract}
In this paper, we study \emph{Federated Bandit}, a decentralized Multi-Armed Bandit problem with a set of $N$ agents, who can only communicate their local data with neighbors described by a connected graph $G$. Each agent makes a sequence of decisions on selecting an arm from $M$ candidates, yet they only have access to local and potentially biased feedback/evaluation of the true reward for each action taken. Learning only locally will lead agents to sub-optimal actions while converging to a no-regret strategy requires a collection of distributed data. Motivated by the proposal of federated learning, we aim for a solution with which agents will never share their local observations with a central entity, and will be allowed to only share a private copy of his/her own information with their neighbors. We first propose a decentralized bandit algorithm \texttt{Gossip\_UCB}, which is a coupling of variants of both the classical gossiping algorithm and the celebrated Upper Confidence Bound (UCB) bandit algorithm. We show that \texttt{Gossip\_UCB} successfully adapts local bandit learning into a global gossiping process for sharing information among connected agents, and achieves guaranteed regret at the order of $O(\max\{ \texttt{poly}(N,M) \log T, \texttt{poly}(N,M)\log_{\lambda_2^{-1}} N\})$ for all $N$ agents, where $\lambda_2\in(0,1)$ is the second largest eigenvalue of the expected gossip matrix, which is a function of $G$. We then propose \texttt{Fed\_UCB}, a differentially private version of \texttt{Gossip\_UCB}, in which the agents preserve $\epsilon$-differential privacy of their local data while achieving $O(\max \{\frac{\texttt{poly}(N,M)}{\epsilon}\log^{2.5} T, \texttt{poly}(N,M) (\log_{\lambda_2^{-1}} N + \log T) \})$ regret. 
\end{abstract}

\maketitle

\section{Introduction}
When data resides at distributed ends, soliciting them to a single server to perform centralized learning might compromise users' privacy. Among all solutions, federated learning (FL) \cite{yang2019federated,kairouz2019advances} arises as a promising paradigm, where massive users are allowed to collaboratively train a model while keeping the training data decentralized at local. In this paper, we introduce \texttt{federated bandit} with fully decentralized users/decision-makers and heterogeneous rewards. Our aim is to provide a solution to enable collaborative learning among decentralized sequential decision-makers in the classical multi-armed bandit (MAB) setting, but with strong (i) regret guarantee even with heterogeneous reward observations, and (ii) privacy guarantees of each user's local data.

\textbf{Motivations }
We are motivated by a critical federated learning scenario where multiple agents hold different and heterogeneous datasets or locally optimal models for the same task \cite{yang2019federated,wang2020local,Li2020On}.
Assuming heterogeneous datasets or models is more practical and challenging than homogeneous ones.
Intuitively, one homogeneous dataset should be uniformly selected from the whole dataset, while a heterogeneous dataset could be generated by any sampling strategies.
This heterogeneity exists in practice for multi-fold of reasons: it could be because of local observation and data collection errors, or it could be due to sampling biases, but the goal of each agent is to cooperatively smooth out these local biases and learn the true optimal action while protecting its private information.
Note that learning with homogeneous datasets is unbiased even without cooperation while it does not hold for heterogeneous ones.%

\textbf{Examples }
Collaborative research among countries is vital when facing global health emergencies, like COVID-19.
The collected observations of the effectiveness of treatments might encode local biases due to the difference in the training of medical staff, in following the protocols and equipment, and the difference in the underlying diseases of patients, etc. Therefore, the observed effect buries noise and local biases. 
For another example, the local datasets of different federated agents may be quite different. Consider an image classification task where agent $1$ holds a dataset with nearly all images being labeled as ``CAT'' while  agent $2$ holds a dataset dominated by ``DOG''. The training will be biased and the local model of each agent may always predict the corresponding dominate label.
In these cases, sharing local models or estimates helps average out or smooth out the biases to obtain a more accurate model for predictions.

To learn the globally optimal model, there is a minimum communication requirement for the heterogeneous datasets.
Coupled with privacy concerns, the communication (with minimum costs) among devices is challenging in federated learning, especially in our interested fully-decentralized scenario \cite{yang2019federated,kairouz2019advances,li2019federated}.
As the first step to solving this challenging problem, we consider the simplest stochastic bandit scenario and focus on the ``minimum'' communication, consistency, and privacy.%
Specifically, consider the following running example.
Suppose that multiple hospitals decide to test the effectiveness of different treatment plans (\emph{arms}). 
Due to the limitations such as data size, health condition and demographics of the patient population, and the details of how a treatment is used \cite{li2019federated}, each individual hospital may not be able to fully and truthfully observe the effect of treatments. 
In other words, individual hospitals will only observe locally biased feedback on the deployed treatment (\emph{heterogeneous rewards}). Sharing observations across institutes is therefore helpful to improve the decision-making process to make sure each local hospital will not converge to a biased and sub-optimal decision. However, due to privacy regulation, it is hard and expensive to call for centralized efforts to coordinate a transportation of data among hospitals. On the other hand, it is relatively easier for individual hospitals to reach agreements to share their observed treatment plan and effects with several others in an ad hoc way.

The hospital treatment selection problem mentioned above is effectively a sequential decision-making problem which can be abstracted as a MAB one. %
Formally, there is a group of $N$ decision-makers facing a common set of arms. At each step $t = 1,2,\ldots,T$, each decision-maker selects one arm in parallel. Decision-makers only have access to local \emph{biased} rewards. Therefore, the agents' individually observed rewards do not fully reflect the true quality of each arm. Instead, the arms' true rewards are collectively decided by all decision-makers' local observations. 
In our heterogeneous setting, we consider a tractable scenario where the true quality of each treatment (arm) is the average of all hospitals' (agents') locally observed quality (in expectation). Each user aims to select the best arm via exchanging information only with their neighbors privately.

\textbf{Overview of technical challenges }
The key technical challenge of the above learning problem lies in the fully-decentralized information sharing and privacy protection with sequential observations.
Firstly, to reduce the communication overhead and privacy leakage during decentralized information sharing, we desire a solution with agents only sharing the information over the adjacency matrix (graph) of neighbor agents in a gossiping way. Note classical gossiping methods \cite{Boyd2006} do not incorporate individual decision-maker's newly observed reward information. Accordingly, the resulting information at all other agents may not converge and reflect the true statistics of each arm adaptively, which is especially true in the fully-decentralized and heterogeneous settings. Tacking the coupling effects of gossiping and bandit learning is challenging.
Secondly, even though the gossiping update is better than directly sharing data in terms of privacy, we still need a mechanism to ensure a specific privacy level in the worst case by assuming a powerful adversary. We adopt the solution concept differential privacy (DP) \cite{dwork2006calibrating,dwork2014algorithmic,erlingsson2014rappor,ding2017collecting} and extend our gossiping bandit solution to a differentially private one. %

\textbf{Contributions }
In this paper, we attempt to solve the above federated bandit learning problem: (1) We introduce a novel extension of the classical MAB problem to a fully-decentralized federated learning setting with gossiping, where an individual decision-maker only has access to biased rewards, and agents have limited communication capacity and can only exchange their beliefs of rewards with neighbors; (2) We propose \texttt{Gossip\_UCB} to solve the challenges for combining gossiping with bandit learning processes, and develop novel proof techniques to tackle the coupling effects of such combination and guarantee its regret. In particular, we provide a concentration bound for local estimates of arm's reward using gossiping. 
(3) To ensure the differential privacy for each agent, we extend the proposed gossiping bandit algorithm to \texttt{Fed\_UCB}, and prove agents preserve $\epsilon$-differential privacy of their local data while achieving $O(\max \{\frac{\texttt{poly}(N,M)}{\epsilon}\log^{2.5} T, \texttt{poly}(N,M) (\log_{\lambda_2^{-1}} N + \log T) \})$ regret. \texttt{Fed\_UCB} is also tested using real medical dataset \cite{strack2014impact}.
(4) To the best of our knowledge, \texttt{Fed\_UCB} is the first fully decentralized bandit learning framework that handles heterogeneous data sources with a privacy guarantee. The results lay the foundation to study more sophisticated and probably more practical settings (e.g., contextual bandit setting to further handle population biases at each local agent).

\subsection{Related Works}
Most relevant to our paper are three lines of works:
 
\textbf{Distributed MAB }
Recently, MAB problems have been studied within multi-agent settings
\cite{leonard,leonard2020,jain,liu2010distributed,ilai2018distributedbandit,kalathil2014decentralized,tossou2015differentially,dubey2020cooperative,korda2016distributed,jingxuan,sankararaman2019social}. But these works mostly either do not consider a consensus reaching in cheap communication setting (gossiping), or do not target on heterogeneous rewards where agents' observations incorporate local bias.
For example, instead of reaching consensus among agents, \cite{jain,liu2010distributed,ilai2018distributedbandit,kalathil2014decentralized,tossou2015differentially} focused on avoiding the collision in wireless communication or cognitive radio.
The homogeneous rewards were assumed in
\cite{leonard,leonard2020,sankararaman2019social,Wang2020Distributed,chakraborty2017coordinated,martinez2019decentralized,korda2016distributed,jingxuan,sankararaman2019social}. However, the rewards in federated learning setting should be heterogeneous due to various limitations \cite{li2019federated,shen2021AAAI}.
To the best of our knowledge, we are the first paper studying the heterogeneous reward in decentralized MAB problems. At the meantime, we are recently made aware of a parallel contribution which deals with the heterogeneity by designing the client sampling strategy of a central server \cite{shen2021AAAI}.%

\textbf{Information propagation and gossiping }
The idea of gossiping was originally proposed to solve the \emph{consensus reaching}  problem in distributed computation \cite{Boyd2006,Basar2007,Morse2003,Morse2011p,MouraSurvey,MurraySurvey}, and questions about gossiping convergence rate were studied in \cite{Boyd2006,Morse2011p,Ts2009}.
It has also been used to solve distributed problems, such as convex optimization, ranking, and voting problems; and more recently to computing machine learning related statistics. Notable examples include \cite{korada2011gossip} for calculating PCA, \cite{pelckmans2009gossip,colin2015extending} for computing U-Statistics, \cite{sirb2018decentralized,liu2018differentially} for computing gradients, \cite{hegedHus2019gossip} for federated learning, and \cite{romoff2019gossip} for reinforcement learning.

\textbf{Federated learning and privacy preserving bandit }
Due to the high demand for privacy protection across different sectors such as financial, medical, and government systems, federated learning is becoming a trending solution that has been widely discussed \cite{kairouz2019advances,bonawitz2016practical,fl:communication,strack2014impact,yang2019federated,su2019securing}. Recently, differential privacy has also been adopted in solving MAB problems while ensuring privacy \cite{Mishra2014PrivateSM,tossou2016algorithms,Malekzadeh2019PrivacyPreservingB,dubey2020private,li2020federated,dubey2020differentially}, but they either consider a single agent problem or use a homogeneous reward setting. 
Specifically, the authors of \cite{li2020federated} assume homogeneous rewards (different agents hold the same mean of each arm) and extend the non-private UCB1 algorithm to a private version with $(\epsilon, \delta)$-DP guarantee.
A multi-agent DP solution in linear bandits (based on linear UCB) is proposed in \cite{dubey2020differentially}. It also assumes the homogeneous setting since all the agents have the same optimal linear model.%
We will follow the idea of DP in our work and offer a theoretically rigorous treatment for our federated bandit problem.

\section{Problem Formulation}

Consider a network consisting of $N$ agents, we assume $N\ge 3,$ since when $N=2,$ it is equivalent to a centralized model thus solution is trivial.
For ease of presentation, we label the agents from $1$ through $N$.
The agents are not aware of such a global labeling, but can differentiate between their neighbors. The set of agents is denoted by $[N]=\{1,2,\ldots,N\}.$
All agents face a common set of $M$ arms, denoted by $[M]=\{1,2,\ldots,M\}$.
At each discrete time $t\in\{1,2,\ldots,T\}$, each agent $i$ makes a decision on which arm to select
from the $M$ options; the selected arm is denoted by $a_i(t)\in [M]$.
When agent $i$ selects an arm $k\in [M]$, the agent collects a reward which is generated according to a random variable $X_{k}(t)$.\footnote{
Different agents may select the same arm $k$ at same time $t$. If this is the case, their rewards can be different as they may collect different realizations of $X_k(t)$.}
But the agent cannot observe its exact reward; instead,
it observes a locally biased ``noisy'' copy of the reward, which is generated
according to another random variable $X_{i,k}(t)$.
The unobservability of $X_k(t)$ can be due to local observational bias caused by the biased local data or low-quality (irresponsible or even malicious) agents \cite{liu2015online}.
We assume that $\{X_k(t)\}_{t=1}^T$ and $\{X_{i,k}(t)\}_{t=1}^T$ are i.i.d. random processes.
For simplicity of analysis, we also assume that all $X_k$'s and $X_{i,k}$'s have bounded support $[0,1]$.
Besides, the feedback $X_{i,k}$ for agent $i$ is assumed to be obtained without delays since the delayed feedback itself is challenging \cite{cayci2019learning,joulani2013online} and we focus on the delayed impact during gossiping.
The relationship between $X_k(t)$ and $X_{i,k}(t)$ is as follows.
Let $\mu_k$ and $\mu_{i,k}$ be the mean of $X_k(t)$ and $X_{i,k}(t)$, respectively.
For each $k\in [M]$, the mean of arm $k$'s reward equals the average\footnote{It can be generalized to the cases where the global reward is defined as any ``convex combination'' of all local rewards following the ``push sum'' idea \cite{kempe2003gossip}. It can also be defined statistically using expectation \cite{shen2021AAAI}.}
of the means of all
agents' observed rewards,
 i.e.,
 \vspace{-2pt}
$$
{\mu_k := \frac{1}{N}\sum_{i=1}^N \mu_{i,k}},
$$
\vspace{-2pt}
which implies that the true reward can be obtained by averaging and thus cancelling out local biases. 
Note the heterogeneous reward {can model the systematic observation bias or the bias of datasets, which}
is more general and meaningful than the homogeneous reward, especially in FL settings \cite{yang2019federated} {where each agent's systematic observation bias makes the locally optimal solution does not correspond to the real optimal action.
Although the bias of datasets is probably more suitable to be modeled as contextual bandits in practice, we currently focus on the classical bandit setting for a theoretically sound solution, which is also an essential foundation for future practically feasible extensions.}
Without loss of generality, suppose $\mu_1 \geq \mu_2 \geq \cdots \geq \mu_M$,
which implies that arm 1 is the best option. The difference of each arm's mean reward is denoted by $\Delta_k = \mu_1 - \mu_k.$ The {\em federated bandit problem} is for each agent $i$ to minimize the following (weak) {\em regret}:
 \begin{align*}
R_i(T) = T\mu_1 - \sum_{t=1}^T \mathbb E\left[X_{a_i(t)}(t)\right]~,
\end{align*}
with the goal of achieving $R_i(T) = o(T)$ (i.e., $R_i(T)/T \rightarrow 0$ as $T\rightarrow\infty$) for all $i\in [N]$.
It is worth noting that each agent $i$ only observes $X_{i,k}$, $k\in [M]$, and
$\mu_{i,1}$ is {\em not} necessarily the largest among $\mu_{i,1}, \mu_{i,2}, \ldots, \mu_{i,M}$.
A key property of our problem setting, which distinguishes from the existing literature, is that no single agent can learn the optimal decision without communication. The setting thus better motivates cooperative federated learning in peer-to-peer networks.%
The heterogeneous reward structure is ready for extension to linear bandits settings \cite{dubey2020differentially} where agents hold different locally optimal combinations of arms or a broader contextual case \cite{dubey2020kernel,li2010contextual} by involving the context (e.g. the feature or local feature \cite{london2019logarithmic,rubinfeld2011fast} of each example).

\subsection{Privacy Guarantee}

In federated bandits, for a given $T$, each agent would like to preserve the privacy of their observations of the selected arms, i.e. the sequential observations $\{X_{i,k}(t)\}_{t=1}^T$.
Differential privacy (DP) is one popular measure to quantify the privacy level of an algorithm $\mathcal B$ \cite{dwork2006calibrating}.
A DP mechanism can make the adversary hard to distinguish two adjacent streams $\{X_{i,k}(t)\}_{t=1}^T$ and $\{X'_{i,k}(t)\}_{t=1}^T$, which differ at a particular $t$.
Let $\mathcal C$ be the space of all possible outputs by Algorithm $\mathcal B$. DP is defined as:
\begin{Definition} (Differential privacy \cite{dwork2006calibrating})\label{def:DP}
A (randomized) algorithm $\mathcal B$ is $\epsilon$-differentially private if for any adjacent streams $\{X_{i,k}(t)\}_{t=1}^T$ and $\{X'_{i,k}(t)\}_{t=1}^T$, and for all sets $\mathcal O \in \mathcal C$,
\[
 \mathbb P \left[\mathcal B(\{X_{i,k}(t)\}_{t=1}^T) \in \mathcal O\right] \le e^\epsilon \cdot \mathbb P\left[ \mathcal B(\{X'_{i,k}(t)\}_{t=1}^T) \in \mathcal O\right].
\]
\end{Definition}
Moreover, in the online settings, the length of sequential observations, i.e. $T$, changes over time. Thus the differential privacy in federated bandits should be guaranteed on every $T$.

\subsection{Communication Graph}
The neighbor relationships among the agents is described by a simple, undirected, connected graph $G=(V,E)$,
whose vertices correspond to agents and edges characterize neighborhood relationships. Denote $\mathcal N_i$ as the set of nodes that are directly connected to agent $i$. 
Then, $\mathcal N_i$ is also the set of agent $i$'s neighbors. 
We follow the setting in the classical gossiping setting \cite{Boyd2006} that at each time $t$,  exactly one pair of two neighboring agents %
on an edge in $E$ are activated and exchange information.

\section{Gossip UCB}\label{Sec:GossipUCB}

Before we offer the privacy-preserving solution, we first introduce an extension of the classical Upper Confidence Bound algorithm to a gossiping setting. 
As may be noticed, in the classical gossiping setting \cite{Boyd2006}, the consensus is defined over initial data only. While in our setting, not only is the gossiping process required to incorporate with newly arrived data from each agent, but also the gossiped information will affect the arm to be selected and thus the observed data of an agent in the future. 
We present an algorithm, called \texttt{Gossip\_UCB}, to solve the gossiping bandit problem.
The algorithm hinges on combining and extending the classical gossiping algorithm and 
the celebrated UCB1 index policy\cite{auer2002finite}; yet our algorithm requires substantial changes for both the gossiping and bandit learning steps. While we will present the algorithm and analysis in detail, we outline here several crucial steps:
\begin{enumerate}
    \item[\textbf{(1)}] Different from the classical gossiping algorithm, the gossiping procedure will incorporate new information from each agent's local sampling and observations at each step $t$. We modify the classical gossip algorithm by adding the ``\emph{gradient}'' information at each step, which capture the newly arrival information.
    \item[\textbf{(2)}] The coupling effects: compared to standard bandit learning with only one decision maker where the traditional sample complexity bound can be employed, we need to cope with various \emph{uncertainties} during gossiping data exchange. These uncertainties lead to further errors in each local agents' decisions, and then will again disrupt the gossiping process.
    \item[\textbf{(3)}] A \emph{fully-decentralized} structure requires designing a local information sharing mechanism. In addition, the \emph{delayed impact} of local information sharing should be bounded analytically for computing the confidence bound locally.
\end{enumerate}

\subsection{Preliminaries}\label{Sec:notations}

We define several quantities that will help us present our algorithm and analysis smoothly.

\textbf{Sample counts:} Each agent $i$ maintains two sets of counters:\\
\noindent $\bullet$ $n_{i,k}(t)$: the number of times agent $i$ has sampled arm $k$ by time $t$; \\
\noindent $\bullet$ $\tilde{n}_{i,k}(t)$: agent $i$'s local estimate of global maximum of pulls on arm $k$ and is updated as:
\begin{align}
    \tilde{n}_{i,k}(t+1)=\max\{n_{i,k}(t),\tilde{n}_{j,k}(t),j\in \mathcal N_i\}. \label{tilden}
\end{align}
The initial value $\tilde{n}_{i,k}(t)$ is defined as $\tilde{n}_{i,k}(0)=n_{i,k}(0).$
We assume agent $i$ can observe $\tilde{n}_{j,k}(t),\;j\in\scr N_i,$ thus is able to update $\tilde{n}_{i,k}(t)$ at each time.

\textbf{Sample mean:} Let $\mathds 1(\cdot)$ be an indicator function that returns 1 when the specific condition holds and 0 otherwise.
Sample mean {\small$\tilde{X}_{i,k}(t)$} is the average observation of agent $i$ on arm $k$ at time $t$: \begin{equation}\label{Eq:def_tildeX}
    \tilde{X}_{i,k}(t) = \frac{1}{n_{i,k}(t)}{\sum_{\tau=1}^t \mathds 1(a_i(\tau) = k) \cdot X_{i,a_i(\tau)}(\tau)}.
\end{equation}

\textbf{Estimate of rewards:} Each agent $i$ maintains an estimate of the reward of arm $k$ at time $t$, which is supposed to be unbiased and denoted by $\vartheta_{i,k}(t)$. The agents' goal is to narrow the gap between $\vartheta_{i,k}(t)$ and $\mu_k$ with sequential observations and gossiping. 

\textbf{Upper confidence bound:} 
In the UCB algorithm, agent $i$'s belief on each arm $k$ relies on two terms: the estimate $\vartheta_{i,k}(t)$ and the upper confidence bound $C_{i,k}(t)$. The latter term denotes the uncertainty of belief.
The arm to be pulled by agent $i$ is denoted as 
$a_i(t) = \arg\max_k  ~\vartheta_{i,k}(t-1)+C_{i,k}(t)$.

\textbf{Gossiping matrix:} 
Denote the gossiping matrix over $G$ as
$$W:= \frac{1}{|E|}\sum_{(i,j) \in E}\left( I_{N} - \frac{1}{2}(e_i-e_j)(e_i-e_j)^{\top}\right),$$ 
which is a positive semi-definite matrix whose largest eigenvalue equals 1, and its second largest eigenvalue is denoted by $\lambda_2(W)$ and short-handed as $\lambda_2$ without ambiguity.
Note $\lambda_2<1$ whenever $G$ is connected \cite{Boyd2006}.

\begin{figure}
    \centering
    \includegraphics[width=\textwidth]{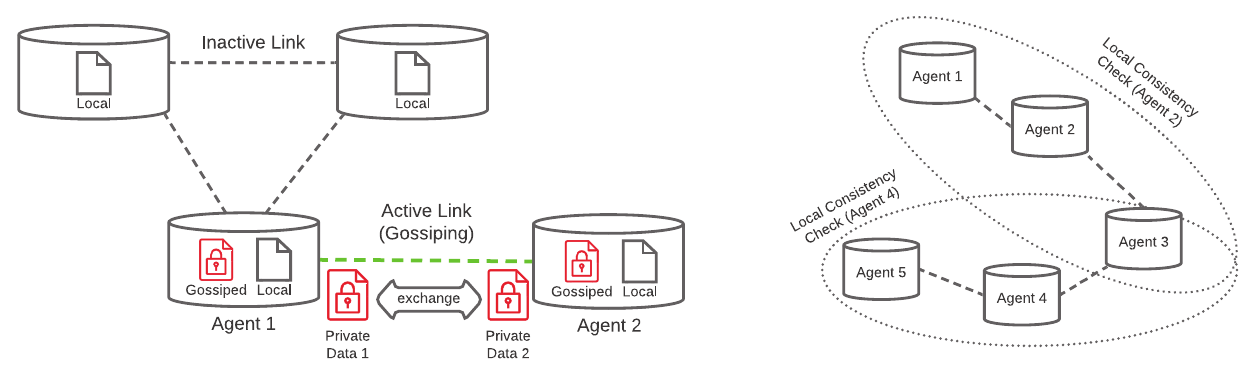}
    \caption{Illustration of our gossiping procedures. At time $t$, only one link is activated for gossiping and (private) data is exchanged via this link. Local consistency is checked by sharing local sample counts.}
    \label{fig:illustrate}
\end{figure}

\subsection{Algorithm}
\texttt{Gossip\_UCB} is detailed in Algorithm \ref{alg:gossip_UCB} and illustrated in Fig.~\ref{fig:illustrate}. Each agent runs this algorithm in parallel. %
Note all the information is shared in a fully distributed fashion and the bandit estimate is updated in a gossiping way. There are two points worth noting.

\textbf{Local information sharing }
Throughout the algorithm, agents need to share two local variables with their neighbors: the local estimate of global maximum of pulls $\tilde n_{i,k}(t)$ and the estimate $\vartheta_{i,k}(t)$.
The sample count $\tilde n_{i,k}(t)$ is shared to keep all the agents ``on the same page''.
Note the bottleneck of a bandit problem is insufficient observations of a particular arm $k$, and the essential of UCB algorithms is encouraging the exploration of these ``undersampled'' arms. 
In the multi-agent scenario, we can take the advantage of neighboring agents and require some local consistency in sample counts. Particularly, we want to keep all agents' knowledge of arm $k$ progressing by encouraging $n_{i,k}(t) \ge \tilde n_{i,k}(t) - N$ in line \ref{line:localN}.
Recall $\tilde{n}_{i,k}(t)$ is agent $i$'s local estimate of global maximum number of pulls, and is updated by local observations only as is defined in the previous section. 
Moreover, we will prove that, this requirement helps \emph{get rid of relying on any global sample count} thus $C_{i,k}(t)$ can be computed by each agent locally.
The estimate $\vartheta_{i,k}(t)$ is updated in a gossiping way.
Besides, compared to existing the decentralized solution \cite{dubey2020differentially} which requires synchronization via broadcasting necessary information at particular checkpoints, our gossiping algorithm is implemented in a more ad-hoc way.%

\textbf{Gossip bandit update }
The gossip updates are defined in line \ref{line:gossip1} and line \ref{line:gossip0}.
In traditional bandit problems, it is enough for each agent to maintain $\tilde X_{i,k}(t)$.
However, in our concerned gossiping setting, solely relying on $\tilde X_{i,k}(t)$ may induce a biased estimate. The gossiping mechanism follows \cite{liu2018differentially}, where the difference $\tilde{X}_{i,k}(t)-\tilde{X}_{i,k}(t-1)$ can be seen as a gradient. Later we will show the effectiveness of the proposed gossip bandit update.

\begin{algorithm}[!t]
\LinesNumbered
\DontPrintSemicolon   
\SetKw{KwRet}{Return}
\KwIn{$ G,T,C_{i,k}(t)$}
\BlankLine
\textbf{Initialization:} Each agent pulls each arm once, and receives a reward $X_{i,k}(0)$, $i\in [N]$, $k\in[M]$. Set $n_{i,k}(0)=1$, $\vartheta_{i,k}(0)=\tilde{X}_{i,k}(0)=X_{i,k}(0).$
\BlankLine
\For{$t = 1,\ldots,T$}{
$\mathcal A_i = \varnothing$ \;
$n_{i,k}(t)=n_{i,k}(t-1), \forall k\in[M]$  \;
$\tilde{n}_{i,k}(t+1)=\max\{n_{i,k}(t),\tilde{n}_{j,k}(t),j\in \scr N_i\},\forall k\in[M]$ \;
Put $k$ into set $\scr{A}_{i}$ ~\textbf{if} ~$n_{i,k}(t)<\tilde{n}_{i,k}(t)-N, \forall k\in[M] $ \label{line:localN}  \hfill  \algcom{// local consistency requirements} \\
\eIf{$\scr{A}_{i}$ is empty}{
\For{$k=1,\dots,M$}{
$Q_{i,k}(t):=\vartheta_{i,k}(t-1)+C_{i,k}(t)$ \label{line:arm1}  \hfill \algcom{// update the belief on each arm}\\
$a_{i}(t)=\arg \max_{k} Q_{i,k}(t)$ \label{line:arm2} \hfill \algcom{// select the best arm to pull}
}}{
$a_{i}(t)$ is randomly selected from $\scr A_{i}$\label{alg:one_arm}}
Observe arm $a_i(t)$, get $X_{i,a_i(t)}(t)$, and update $\tilde X_{i,k}(t), \forall k,$ following (\ref{Eq:def_tildeX})\; $n_{i,a_{i}(t)}:=n_{i,a_{i}(t)}+1$ \;
\eIf{agent $i$ is selected to gossip with agent $j$}{
agent $i$ sends $\vartheta_{i,k}(t-1)$ to agent $j$\;
agent $i$ receives $\vartheta_{j,k}(t-1)$ from agent $j$\;
$\vartheta_{i,k}(t):=\frac{\vartheta_{i,k}(t-1)+\vartheta_{j,k}(t-1)}{2}+\tilde{X}_{i,k}(t)-\tilde{X}_{i,k}(t-1)$ \label{line:gossip1} \hfill  \algcom{// gossiping update}
}{
$\vartheta_{i,k}(t):=\vartheta_{i,k}(t-1)+\tilde{X}_{i,k}(t)-\tilde{X}_{i,k}(t-1)$ \label{line:gossip0} \hfill  \algcom{//  normal update}}
}
\caption{{\texttt{Gossip\_UCB}}}
\label{alg:gossip_UCB}
\end{algorithm}

\subsection{Technical Challenges}
Note the arm selection in Algorithm~\ref{alg:gossip_UCB} relies on the upper confidence bound $C_{i,k}(t)$. A UCB-based solution requires us to find an appropriate choice of $C_{i,k}(t)$ that allows us to provide guaranteed bound of each agent's regret when implementing \texttt{Gossip\_UCB}.

The main technical challenge is tackling the \emph{coupling effects of gossiping and bandit learning}.
On a high level, classical technical results in gossiping assumed a \emph{static} piece of information that would not change much during the entire gossiping phase. 
The literature \cite{szorenyi2013gossip} often adopted a phased-based learning strategy (by caching the gossiped information) to avoid changes, which effectively delays the update of learned policy (the learning needs to wait for the gossiping to converge).
Since agents only share information with their neighbors, globally, there is latency in receiving this data at the non-directly connected agents. 
Additionally, with the existence of multiple agents, ensuring local consistency as lines \ref{line:localN} and \ref{alg:one_arm} will incur extra delay impacts when $|\mathcal A_i|>1$.
The \emph{delayed impact} affects the immediate decisions taken by other agents and further affects the gossiping process in the near future. This is a salient challenge when the agents make heterogeneous observation, as compared to the homogeneous setting \cite{sankararaman2019social}.
From the bandit learning's perspective, this delay might lead to inaccurate computation of the index policies. A poorly made decision will further have cascading effects to other receiving agents, which is especially challenging in heterogeneous settings.
The key step to tackle this challenge is to firstly characterize the consistency among agents and then find an concentration bound for $\vartheta_{i,k}(t)$, the local estimate. %

\subsection{Guarantees on the Consistency among Agents}\label{sec:consistency}

Intuitively speaking, the minimum number of selection over agents controls the quality of the average statistics. 
Recall $\tilde n_{i,k}(t)$ denotes agent $i$'s local estimate of global maximum of pulls on arm $k$.
In line \ref{line:localN}, we encourage $n_{i,k}(t) \ge \tilde n_{i,k}(t) - N$
to make sure all agents' knowledge of arm $k$ is ``on the same page'', i.e. locally consistent. This consistency requirement is essential in multi-agent scenarios with heterogeneous rewards to make sure that local ``estimates'' of the number of arm selections are boundedly consistent across the network.
Particularly, with local consistency, the gap of sample counts among neighboring agents, i.e. $n_{j,k}(t), j\in \mathcal N_i$ could be bounded.
Moreover, the local consistency implies global consistency, i.e., the number of arm selections performed by each agent should be consistent globally,
such that the local estimates of each agent should be precise enough to derive a confidence bound without global information.
This further enables a proof of a concentration bound for $\vartheta_{i,k}(t)$ in Section \ref{sec:concentration_bound}.

However, due to the limitation of the bandit feedback (one arm each time) as shown in line \ref{alg:one_arm}, extra delays will occur thus $n_{i,k}(t) \ge \tilde n_{i,k}(t) - N$ \emph{cannot be guaranteed} when $|\mathcal A_i|>1$.
These extra delays will be further amplified due to the delay of information propagation via $G$.
In the following of this subsection, we will analyze the local consistency and show how agents' knowledge of $n_{i,k}$s will remain ``consistent".
The analyses are based on the information propagation on the neighbor graph $G$, where we use $d_{i,j}$ to characterize the distance (the number of directed edges in the shortest directed path) from vertex $i$ to vertex $j$ in graph $G$.
Note $d_{i,i}=0$ and $d_{i,j}<N$ since $G$ is strongly connected.
W.l.o.g., let $n_{i,k}(t)=0, \forall i\in[N], k\in[M]$ when $t<0$.
The analyses can be summarized into \emph{three steps}: modeling the information propagation (Step 1), bounding the actually achieved local consistency (Step 2), and characterizing the relationship of knowledge among agents in the fully-decentralized setting (Step 3).
\subsubsection*{\textbf{Step \uppercase\expandafter{\romannumeral1}: Information propagation:}}~\\ %
As in (\ref{tilden}), $\tilde{n}_{i,k}(t+1):=\max\{n_{i,k}(t),\tilde{n}_{j,k}(t),j\in \mathcal N_i\}$ is defined recursively. Considering the fact that each agent $i$ can observe $\tilde{n}_{j,k}(t),\;j\in\scr N_i$ and update $\tilde{n}_{i,k}(t)$ at each time $t$, we know the information $n_{i,k}(t)$ propagates distance $1$ when time goes from $t$ to $t+1$. Thus it is possible to model the information propagation with distance $d_{i,j}$. We use Lemma~\ref{max} to characterize the delayed impact of information propagation.
\begin{Lemma}\label{max}
{(Information propagation)}
For any $i\in[N]$ and $k\in[M]$,
\begin{align}
    \tilde{n}_{i,k}(t+1) = \max_{j\in[N]}\left\{n_{j,k}(t-d_{j,i})\right\}.\label{maxdis}
\end{align}
\end{Lemma}

\begin{proof}
We will prove the lemma by induction on $t$.
For the basis step, suppose that $t=0$.
In this case, $\tilde{n}_{i,k}(1)=\max\{n_{i,k}(0),\;\tilde{n}_{j,k}(0),\;j\in\scr{N}_i\}=1.$ Note that $\max_{j\in[N]}\{n_{j,k}(t-d_{j,i})\}=n_{i,k}(0)=1.$ Thus, \rep{maxdis} holds when $t=0.$

For the inductive step, assume \eqref{maxdis} holds at time $t,$ and now consider time $t+1.$ 
Note that
\begin{align*}
    \Tilde{n}_{i,k}(t+1)&=\max\{n_{i,k}(t),\;\tilde{n}_{j,k}(t),\;j\in\scr N_i\}\\
    &=\max\{n_{i,k}(t),\;n_{h,k}(t-d_{j,h}-1), h\in[N],\;j\in\scr N_i\}.
\end{align*}
Noting $d_{h,i}$ is the distance from vertex $h$ to vertex $i$, the following inequality holds: $d_{h,i}\leq d_{j,i}+d_{h,j}=1+d_{h,j}$.
Since $n_{i,k}(t)$ is a non-decreasing function of $t$ by its definition, we have
\begin{align}
     \Tilde{n}_{i,k}(t+1)&\leq \max\{n_{i,k}(t),\;n_{h,k}(t-d_{h,i}),\;h\in[N]\}=\max_{j\in [N]}\{n_{j,k}(t-d_{j,i})\}.\label{leqmax}
\end{align}
Fix any vertex $j\in[N]$ and let $p=(j,v_{d_{j,i}},\ldots,v_2,i)$ be a shortest directed path from $j$ to $i$ in $G$. From \eqref{tilden},
\begin{equation}\label{distance}
    \begin{split}
        \tilde{n}_{i,k}(t+1)&\geq \tilde{n}_{v_2,k}(t)\geq\cdots\geq\tilde{n}_{v_{d_{j,i}},k}(t-d_{j,i}+2)\geq\tilde{n}_{j,k}(t-d_{j,i}+1)\geq n_{j,k}(t-d_{j,i}).
    \end{split}
\end{equation}
Since $j$ is arbitrarily chosen from $[N],$ we have $\tilde{n}_{i,k}(t+1)\geq \max_{j\in[N]}\{n_{j,k}(t-d_{j,i})\}.$
Combining with \eqref{leqmax}, we have 
\begin{align*}
    \tilde{n}_{i,k}(t+1)= \max_{j\in[N]}\{n_{j,k}(t-d_{j,i})\}.
\end{align*}
Therefore, equation \eqref{maxdis} also holds at $t+1,$ which completes the induction.
\end{proof}

\subsubsection*{\textbf{Step \uppercase\expandafter{\romannumeral2}: Actual local consistency:}}~\\
With explicit formulation of information propagation in Lemma~\ref{max}, we are ready to present Lemma~\ref{relation1}, showing the actual local consistency achieved by Algorithm~\ref{alg:gossip_UCB}.
This lemma implies that the local selection of an arm $k$ by agent $i$ is not too far away from his/her neighbors' selections. It makes sure agents' selections of a particular arm are consistent so that the error bound in the joint estimation will not be throttled by one particular agent. 

\begin{Lemma}\label{relation1}
{(Actual local consistency)}
$\forall i\in[N], k\in[M]$, we have  $n_{i,k}(t)>\tilde{n}_{i,k}(t+1)-3MN$. 
\end{Lemma}
\begin{proof}
We will prove the lemma by contradiction.
Suppose that, to the contrary, $\exists\; i,k_1$ such that $n_{i,k_1}(t)\leq \tilde{n}_{i,k_1}(t+1)-3MN.$ Let $t'$ denote the first time at which the equality holds, i.e.,
$$n_{i,k_1}(t')=\tilde{n}_{i,k_1}(t'+1)-3MN.$$
The critical value $t'$ must exist. On one hand, when $t=0,$ we have $n_{i,k}(0)>\tilde{n}_{i,k}(1)-3MN$. On the other hand, since both $n_{i,k}(t)$ and $\tilde{n}_{i,k}(t)$ increase by 0 and 1 at each time instance, if there exists some $t$ such that $n_{i,k_1}(t)< \tilde{n}_{i,k_1}(t+1)-3MN,$ there must exist a $t'$ between 0 and $t,$ such that $n_{i,k_1}(t')=\tilde{n}_{i,k_1}(t'+1)-3MN.$
According to Lemma~\ref{max}, $\exists\;j\in[N]$ such that
\begin{align}
    \tilde{n}_{i,k_1}(t'+1)=n_{j,k_1}(t'-d_{j,i}).\label{eq1}
\end{align}
Then,
\begin{align*}
    n_{j,k_1}(t'-d_{j,i})-n_{i,k_1}(t')=3MN,
\end{align*}
and $t'$ is the earliest time instant at which $n_{j,k_1}(t'-d_{i,j})-n_{i,k_1}(t')\geq3MN$ holds. This implies that at time $t'-d_{i,j},$ agent $j$ pulls arm $k_1.$

Since each agent must pull an arm at each time, we have $\sum_kn_{i,k}(t)=t, \;\forall i\in[N].$ Then,
\begin{align*}
    \sum_{k\in[M]\backslash k_1}n_{i,k}(t')-\sum_{k\in[M]\backslash k_1}n_{j,k}(t'-d_{j,i})=3MN+d_{j,i}.
\end{align*}
Applying the Pigeonhole principle, $\exists\; k_2\in[M]$ such that
\begin{align*}
    n_{i,k_2}(t')-n_{j,k_2}(t'-d_{j,i})\geq \frac{3MN+d_{j,i}}{M-1}>3N.
\end{align*}
According to the definition of $n_{i,k}(t),$ it is non-decreasing and $n_{i,k}(t+1)\leq n_{i,k}(t)+1.$ Thus,
\begin{align*}
    n_{i,k_2}(t'-2d_{j,i})-n_{j,k_2}(t'-d_{j,i})&>3N-2d_{j,i}>N.
\end{align*}
Using \eqref{distance}, we have $\tilde{n}_{j,k_2}(t'-d_{j,i}+1)\geq n_{i,k_2}(t'-d_{j,i}-d_{i,j}).$ Thus,
\begin{align*}
    \tilde{n}_{j,k_2}(t'-d_{j,i}+1)-n_{j,k_2}(t'-d_{j,i})>N.
\end{align*}
From the above analysis, agent $j$ must pull arm $k_1$ at time $t'-d_{j,i}$. According to the decision making step of the algorithm, 
there holds
\begin{align}
    \tilde{n}_{j,k_1}(t'-d_{j,i}+1)-n_{j,k_1}(t'-d_{j,i})\geq N>0.\label{eq2}
\end{align}
Note that from \eqref{distance}, 
\begin{align}
    \tilde{n}_{i,k_1}(t'+1)\geq \tilde{n}_{j,k_1}(t'-d_{j,i}+1).\label{eq3}
\end{align}
Combining \eqref{eq1} -- \eqref{eq3} together, we have
\begin{align*}
    n_{j,k_1}(t'-d_{j,i})&=\tilde{n}_{i,k_1}(t'+1)\geq \tilde{n}_{j,k_1}(t'-d_{j,i}+1)>n_{j,k_1}(t'-d_{j,i}),
\end{align*}
which is a contradiction.
Therefore, the statement of the lemma is true.
\end{proof}

\subsubsection*{\textbf{Step \uppercase\expandafter{\romannumeral3}: Global consistency:}}~\\
The propagation model in Lemma~\ref{max} and the local consistency guaranteed in Lemma~\ref{relation1} jointly imply the global consistency among agents.
Mathematically, we quantify the global consistency among agents in Lemma~\ref{relation2}.
This lemma effectively shows that any agent will not pull one arm more than twice of other agents.
Moreover, this result enables a fully-decentralized gossiping structure where no global information of $n_{i,k}(t)$ is needed for local arm selections in lines~\ref{line:arm1}--\ref{line:arm2}.
This is due to the fact that the UCB-based algorithms rely on $n_{i,k}(t)$ for computing the confidence bound, especially for the multi-agent case with heterogeneous rewards where knowledge of $n_{i,k}(t),\forall i\in[N]$ are necessary for each agent if their global consistency cannot be properly guaranteed.

\begin{Lemma}\label{relation2}({Global consistency})
$\forall i\in[N],k\in[M]$, when $n_{i,k}(t)\geq(3M+1)N$, we have $$\max_{j\in[N]} n_{j,k}(t)\leq 2n_{i,k}(t).$$ 
\end{Lemma}

\begin{proof}
From Lemma~\ref{max}, we know
$\tilde{n}_{i,k}(t+1)\geq n_{h,k}(t-d_{h,i}),\;\forall h\in[N].$
Combining with $n_{i,k}(t+1)\leq n_{i,k}(t)+1,$ we have
\begin{align*}
    \tilde{n}_{i,k}(t+1)\geq n_{h,k}(t)-d_{h,i}\geq n_{h,k}(t)-N.
\end{align*}
With the inequality  $n_{i,k}(t)>\tilde{n}_{i,k}(t+1)-3MN$ from Lemma~\ref{relation1}, we further have
\begin{align*}
    n_{i,k}(t)\geq n_{h,k}(t)-(3M+1)N,\;\forall \;h\in[N].
\end{align*}
Since $h$ is arbitrarily chosen and $n_{i,k}(t)\geq(3M+1)N$, we have
\begin{align*}
    \max_{j\in[N]} n_{j,k}(t)&\leq n_{i,k}(t)+(3M+1)N\leq 2n_{i,k}(t),
\end{align*}
which completes the proof.
\end{proof}

\subsection{Concentration Bound for Local Estimates}\label{sec:concentration_bound}

With the global consistency elaborated in Lemma~\ref{relation2}, the concentration bound for local estimates $\vartheta_{i,k}(t)$ could be derived locally.
Different from the settings with homogeneous rewards where the gossiping procedures always speedup the convergence, in the concerned heterogeneous reward scenarios, the gossiped information may induce selecting a non-optimal arm when it comes from an agent with ``sufficiently biased'' reward expectation.
Therefore, we need to carefully analyze the coupling effect of gossiping and bandit learning on $\vartheta_{i,k}(t)$.

We study how $\vartheta_{i,k}(t)$ converges to $\mu_{k}$ as time increases in the remaining of this section.
Mathematically, we need to find an small enough $C_{i,k}(t)$ such that $\mathbb P(|\vartheta_{i,k}(t)-\mu_{k}|\geq C_{i,k}(t))$ could be bounded in a sufficiently small order, e.g. $O(\frac{1}{t^2})$.
Note 
\begin{align*}
    \mathbb P\left(|\vartheta_{i,k}(t)-\mu_{k}|\geq C_{i,k}(t)\right)     
    \le \mathbb P\left(|\vartheta_{i,k}(t)-\mathbb{E}(\vartheta_{i,k}(t))|+|\mathbb{E}(\vartheta_{i,k}(t))-\mu_{k}|\geq C_{i,k}(t)\right).
\end{align*}
We next bound $|\vartheta_{i,k}(t)-\mathbb{E}(\vartheta_{i,k}(t))|$ and $|\mathbb{E}(\vartheta_{i,k}(t))-\mu_{k}|$, respectively.

\subsubsection*{\textbf{Step \uppercase\expandafter{\romannumeral1}: Bound for $|\vartheta_{i,k}(t)-\mathbb{E}(\vartheta_{i,k}(t))|$}:}~\\
Lemma~\ref{concentration1} summarizes this bound. 
We provide the main idea in the following sketched proof and leave the full proof in Appendix~\ref{sup_proxy}.
Note $p_0$ in Lemma~\ref{concentration1} is effectively a function of $i,k,t$. We use the notation $p_0$ for a clean presentation when there is no confusion. 
\begin{Lemma}\label{concentration1}
When $n_{i,k}(t)\geq \max\{L,(3M+1)N\}, \forall i\in[N],$ with probability at least $1-p_0,$ $|\vartheta_{i,k}(t)-\mathbb{E}(\vartheta_{i,k}(t))|$ bounds as follows:
\begin{align*}
    & \mathbb P\left(|\vartheta_{i,k}(t)-\mathbb{E}[\vartheta_{i,k}(t)]|\geq \sqrt{\frac{2N}{n_{i,k}(t)}\log t}\right)  <\frac{2}{t^2},
\end{align*}
where 
$p_0=
\frac{\lambda_{2}^{{n_{i,k}(t)}/{12}}}{1-\lambda_{2}^{{1}/{3}}}
,$ $L$ is the minimal value that satisfies
${\lambda_{2}^{{t}/{6}}}/{(1-\lambda_{2}^{{1}/{3}})}<{(N t)^{-1}}$, $\forall t\geq L.$ 
\end{Lemma}

\begin{proof}
(Sketch)
Note $\vartheta_{i,k}(t)$ is a sub-Gaussian random variable.
Preliminaries of sub-Gaussian random variables can be found in Appendix~\ref{subgaussian}. 
We analyze the coupling effect of gossiping and bandit learning with the variance proxy of $\vartheta_{i,k}(t)$.
Each $\vartheta_{i,k}(t)$ is a linear combination of $X_{j,k}(\tau)$, for all  $j\in [N], \tau\in\{0,\ldots,t\}.$ Define $c_{i,k,j}^{(\tau)}$ as the corresponding coefficient of such  $X_{j,k}(\tau)$ in $\vartheta_{i,k}(t).$ To find the variance proxy of $\vartheta_{i,k}(t),$ which is $\frac{1}{4}\sum_{j=1}^{N}\sum_{\tau=1}^{t}|c_{i,k,j}^{(\tau)}|^2$ according to Property~\ref{sufficientcondition} and Property~\ref{additivity} in Appendix \ref{subgaussian}, we will estimate the value of $c_{i,k,j}^{(\tau)}$ in the following three steps and derive the bound for $|\vartheta_{i,k}(t)-\mathbb{E}(\vartheta_{i,k}(t))|$ using the property of sub-Gaussian random variables.

\paragraph{\textbf{Step \uppercase\expandafter{\romannumeral1}-1: Preparing a particular $c_{i,k,j}^{(0)}$:}}~\\
Note the vector form of the iterative process is:
\begin{align}\label{eq:iterative}
    \vartheta_{k}(t)=W(t)\vartheta_{k}(t-1)+\Tilde{X}_{k}(t)-\Tilde{X}_{k}(t-1),
\end{align}
for all $k\in [M],$ where $W(t)=I-\frac{1}{2}(e_i-e_j)(e_i-e_j)^{\top}$ is the gossip updating matrix when agent $i$ and $j$ are selected to exchange information at time $t$. 
With this recursion formula, we can expand $\vartheta_{k}(t)$ and extract the corresponding coefficients as
 \begin{equation*}
     \begin{split}
         c_{i,k,j}^{(0)}=\left[W(t)\cdots W(\tau_{1}+1)-\sum_{h=2}^{n_{j,k}(t)}\frac{W(t)\cdots W(\tau_{h}+1)}{(h-1)h}\right]_{i,j},
     \end{split}
 \end{equation*}
 where $\tau_{1},\ldots, \tau_{n_{j,k}(t)}$ denote the time instance before $t$ when agent $j$ pulls arm $k$, and $\tau_{1}=0$.
 
\paragraph{\textbf{Step \uppercase\expandafter{\romannumeral1}-2: Upper bound for $c_{i,k,j}^{(0)}$:}}~\\
We need to understand $W(t)\cdots W(\tau_{h}+1)$ to derive the upper bound and the lower bound. One important tool is the convergence bound of randomized gossip algorithms in \cite{Boyd2006}, with which we have
 \begin{align*}
     \mathbb P\left(\left|[W(t)\cdots W(\tau_{h}+1)]_{i,j}-\frac{1}{N}\right|>\lambda_{2}^{\frac{1}{3}(t-\tau_{h})}\right)<\lambda_{2}^{\frac{1}{3}(t-\tau_{h})},
 \end{align*}
 where $\lambda_2 \in (0,1)$ is the second largest eigenvalue of the expected gossip matrix $W$. 
 Let  $p_0={\lambda_{2}^{{n_{i,k}(t)}/{12}}}/{(1-\lambda_{2}^{{1}/{3}})}.$
 With probability at least $1-p_0$, we can get the following upper bound:
 \begin{align*}
c_{i,k,j}^{(0)} <\frac{1}{N}\frac{2}{n_{j,k}(t)}+\frac{\lambda_{2}^{{n_{j,k}(t)}/{6}}}{1-\lambda_{2}^{{1}/{3}}}.
\end{align*}

\paragraph{\textbf{Step \uppercase\expandafter{\romannumeral1}-3: Lower bound for $c_{i,k,j}^{(0)}$:}}~\\
Accordingly, with probability at least $1-p_0$, the lower bound is 
 \begin{align*}
     c_{i,k,j}^{(0)} >&\frac{1}{N}\frac{2}{n_{j,k}(t)}-\frac{\lambda_{2}^{{n_{j,k}(t)}/{6}}}{1-\lambda_{2}^{{1}/{3}}}-\frac{1}{n_{j,k}(t)}.
 \end{align*}

\paragraph{\textbf{Step \uppercase\expandafter{\romannumeral1}-4: Generalization:}}~\\
 Let $L$ be the smallest value that makes ${\lambda_{2}^{{t}/{6}}}/{(1-\lambda_{2}^{{1}/{3}})}<{(N t)^{-1}}$ hold for all $t\geq L.$ Then we have $|c_{i,k,j}^{(0)}|<{1}/{n_{j,k}(t)}$ when $n_{j,k}(t)\geq L$. 
 We can prove the other terms $c_{i,k,j}^{(\tau_{h})}, \forall h\leq n_{j,k}(t)$  also have this property. Thus when $n_{i,k}(t)\geq \max\{L,(3M+1)N\}, \forall i$,\footnote{Otherwise, we have $n_{i,k}(t)<\max\{L,(3M+1)N\}$ and bound the regret by $\Delta_i \max\{L,(3M+1)N\}$.} with probability at least $1-p_0$, we have
\begin{align*}
\sum_{j=1}^{N}\sum_{h=1}^{n_{j,k}(t)}|c_{i,k,j}^{(\tau_{h})}|^2&<\sum_{j=1}^{N}\frac{n_{j,k}(t)}{n^2_{j,k}(t)}<\frac{2N}{n_{i,k}(t)}. 
 \end{align*}
 Note the last inequality holds due to Lemma~\ref{relation2}, which bounds the gap of trials among agents and removes the dependency on global information.
Therefore, the variance proxy (parameter in sub-Gaussian random variable) is bounded by:
\begin{align*}
    \sigma^2_{i,k}(t)&=\frac{1}{4}\sum_{j=1}^{N}\sum_{h=1}^{n_{j,k}(t)}|c_{i,k,j}^{(\tau_{h})}|^2 <\frac{N}{2n_{i,k}(t)}.
\end{align*}
According to the Hoeffding's inequality for sub-Gaussian random variables (specified in Property \ref{hoeffding} of Appendix \ref{subgaussian}),
the bound for $|\vartheta_{i,k}(t)-\mathbb{E}(\vartheta_{i,k}(t))|$ is
(with probability at least $1-p_0$)
\begin{align*}
    & \mathbb P\left(|\vartheta_{i,k}(t)-\mathbb{E}[\vartheta_{i,k}(t)]|\geq \sqrt{\frac{2N}{n_{i,k}(t)}\log t}\right)     \leq  2\exp{\frac{\frac{-2N}{n_{i,k}(t)}\log t}{2\sigma^2_{i,k}(t)}}<\frac{2}{t^2}.
\end{align*}
\end{proof}

\subsubsection*{\textbf{Step \uppercase\expandafter{\romannumeral2}: Bound for $|\mathbb{E}(\vartheta_{i,k}(t))-\mu_{k}|$:}}~\\
In traditional homogeneous reward settings, each agent's estimate is unbiased.
However, in the challenging heterogeneous reward settings, the local estimate of agent $i$ toward arm $k$ at time $t$, i.e. $\vartheta_{i,k}(t)$, may be biased. 
This bias needs to be bounded in analyzing the regret upper bound for \texttt{Gossip\_UCB}.
We summarize the upper bound for this bias in Lemma~\ref{lem:bias}.

\begin{Lemma}\label{lem:bias}
With probability $1$,
$
|\mathbb{E}(\vartheta_{i,k}(t))-\mu_{k}| \le \sqrt{N}\lambda_2^t, \forall i\in [N], k\in [M].
$
\end{Lemma}
\begin{proof}(Sketch)
Taking expectation on both sides of (\ref{eq:iterative}) yields $\mathbb{E}(\vartheta_{k}(t))=W^t\mathbb{E}(\bm \vartheta_{k}(0))$.
Noting $\1$ is an eigenvector of $W$ and $\mu_k = \frac{1}{N} \1^\top \cdot \mathbb{E}(\bm \vartheta_{k}(0))$, the proof can be accomplished by splitting the effect of gossiping and $\mu_k$ with H\"{o}lder's inequality.
\end{proof}

\subsubsection*{\textbf{Wrapping up: concentration bound for local gossiping estimates}}~\\
When $n_{i,k}(t)\geq \max\{L,(3M+1)N\},\forall i \in[N],$ we have 
$$\sqrt{N}\lambda_2^t<\frac{64}{N^{17}}\triangleq \alpha_1.$$
According to Lemma~\ref{concentration1} and Lemma~\ref{lem:bias}, 
set $$C_{i,k}(t)=\sqrt{\frac{2N}{n_{i,k}(t)}\log t}+\alpha_1,$$ then with probability at least $1-p_0,$
the concentration bound for $\vartheta_{i,k}(t)$ is 
\begin{align}\label{concentration_whole}
    \mathbb P(|\vartheta_{i,k}(t) - \mu_{k}|\geq C_{i,k}(t))<\frac{2}{t^2}.
\end{align}

\subsection{Regret Upper Bound for \texttt{Gossip\_UCB}}\label{eq:secRegretGossip}
The concentration bound in (\ref{concentration_whole}) enables the analyses to the regret upper bound for \texttt{Gossip\_UCB}.
We have the following regret upper bound:
\begin{Theorem} {({Regret upper bound for \texttt{Gossip\_UCB}})}
For the \texttt{Gossip\_UCB} algorithm with bounded reward over $[0,1],$ and 
\begin{equation}\label{Eq:bound_clean}
    C_{i,k}(t)= \sqrt{\frac{2N}{n_{i,k}(t)}\log t}+\alpha_1,
\end{equation}
the regret of each agent $i$ until time $T$ satisfies
{
\begin{align*}
    &R_i(T)<\sum_{\Delta_k>0}\Delta_k\left(\max\left\{\frac{2N}{(\frac{1}{2}\Delta_k-\alpha_1)^2}\log T ,L,(3M+1)N\right\}+\alpha_2\right),
\end{align*}}
where $\alpha_1=\frac{64}{N^{17}}$, $\alpha_2=(3M-1)N + \frac{2\pi^{2}}{3}+\frac{2\lambda_2^{1/12}}{(1-\lambda_{2}^{{1}/{3}})(1-\lambda_{2}^{{1}/{12}})}$. 
\label{thm:main}
\end{Theorem}

\begin{proof}
With Section~\ref{sec:consistency} and Section~\ref{sec:concentration_bound} as building blocks, we are ready to present the regret analyses. 
Following the concentration bound for $\vartheta_{i,k}(t)$, we set $C_{i,k}(t)$ as the confidence bound.
According to the standard analyses for UCB algorithm \cite{auer2002finite}, we know
\[
\mathbb E[n_{i,k}(T)] \le n_{i,k}(t') + \sum_{t=1}^T \mathbb E\left[\1 \{
a_i(t) = k, n_{i,k}(t) \ge n_{i,k}(t') \}\right].
\] 
If agent $i$ chooses arm $k$ instead of arm $1$ at time $t$, there are only four possible cases:  1) $k\in \scr{A}_{i}$; 2) $\vartheta_{i,k}(t)-\mu_{k}\geq C_{i,k}(t)$; 3) $\mu_{1}-\vartheta_{i,1}(t)\geq C_{i,1}(t)$; 4) $\mu_{1}-\mu_{k}< 2C_{i,k}(t)$.
Define $t'$ as the time when $n_{i,k}(t')$ is large enough such that Case 4 does not hold.
Specifically, we have
$$n_{i,k}(t'):= \frac{2N}{(\frac{1}{2}\Delta_k-\alpha_1)^2}\log T.$$
Next, we need to bound the non-optimal selections after a sufficiently large time.
Case 2 has two settings: $\vartheta_{i,k}(t)$ is in the concentration bound (Case 2A) and $\vartheta_{i,k}(t)$ is out of the concentration bound (Case 2B).
For Case 2A, we have 
\[
\sum_{t=t'}^T
\mathbb E\left[\1 \{
a_i(t)  = k, n_{i,k}(t) \ge n_{i,k}(t') | \text{Case 2A}\}\right] \le \sum_{t=t'}^T \frac{1}{t^2} \le \frac{\pi^2}{6}.
\]
For Case 2B, since the term $\1 \{
a_i(t) = k, n_{i,k}(t) \ge n_{i,k}(t') \}=1$ only when $n_{i,k}(t) = n_{i,k}(t-1)+1$,
 we directly count one non-optimal selection when it happens, i.e.
\begin{align*}
    & \sum_{t=t'}^T
\mathbb E\left[\1 \{
a_i(t)  = k, n_{i,k}(t) \ge n_{i,k}(t') | \text{Case 2B}\}\right] \le&  \sum_{n'> n_{i,k}(t')}^{n_{i,k}(T)} \frac{\lambda_{2}^{{n'}/{12}}}{1-\lambda_{2}^{{1}/{3}}}  \le \frac{\lambda_2^{{n_{i,k}(t')}/12}}{(1-\lambda_{2}^{{1}/{3}})(1-\lambda_{2}^{{1}/{12}})}. 
\end{align*}
Accordingly, we have Case 3A when $\vartheta_{i,1}(t)$ is concentrated and Case 3B when $\vartheta_{i,1}(t)$ is not. 
The corresponding expectation is bounded as follows.
\begin{align*}
\sum_{t=t'}^T
\mathbb E\left[\1 \{
a_i(t)  = k, n_{i,k}(t) \ge n_{i,k}(t') | \text{Case 3A}\}\right] &\le \frac{\pi^2}{6}.\\
     \sum_{t=t'}^T
\mathbb E\left[\1 \{
a_i(t)  = k, n_{i,k}(t) \ge n_{i,k}(t') | \text{Case 3B}\}\right]  &\le \frac{\lambda_2^{{n_{i,1}(t')}/12}}{(1-\lambda_{2}^{{1}/{3}})(1-\lambda_{2}^{{1}/{12}})}. 
\end{align*}

Considering the gap between the actual local consistency ($n_{i,k}(t)>\tilde{n}_{i,k}(t+1)-3MN$) and the required local consistency ($n_{i,k}(t) \ge \tilde n_{i,k}(t) - N$), the additional non-optimal selections caused by Case 1 is no larger than $\frac{\pi^{2}}{3}+(3M-1)N$.

Let $n_i(t'):=\min_{k\in[M]} n_{i,k}(t')$.
Together we have
\begin{align*}
    \sum_{t=t'}^T
\mathbb E\left[\1 \{
a_i(t)  = k, n_{i,k}(t) \ge n_{i,k}(t') \}\right]
\leq \frac{2\pi^{2}}{3}+\frac{2\lambda_2^{n_i(t')/12}}{(1-\lambda_{2}^{{1}/{3}})(1-\lambda_{2}^{{1}/{12}})}+(3M-1)N.
\end{align*}
Note $n_i(t')>1$ holds trivially.
Therefore, the regret of agent until time $T$ satisfies 
{\small
\begin{align*}
    &R_i(T)    = \sum_{\Delta_k>0}\Delta_k\cdot \mathbb{E}[n_{i,k}(T)]\\
    <&\sum_{\Delta_k>0}\Delta_k\bigg(\underbrace{\max\left\{\frac{2N}{(\frac{1}{2}\Delta_k-\alpha_1)^2}\log T ,L,(3M+1)N\right\}}_{\text{Case 4}}+\underbrace{(3M-1)N +\frac{\pi^{2}}{3}}_{\text{Case 1}}
    +\underbrace{\frac{\pi^{2}}{3}
    +\frac{2\lambda_2^{1/12}}{(1-\lambda_{2}^{{1}/{3}})(1-\lambda_{2}^{{1}/{12}})}}_{\text{Case 2\&3}}\bigg).
\end{align*}}
Let $\alpha_2=(3M-1)N + \frac{2\pi^{2}}{3}+\frac{2\lambda_2^{1/12}}{(1-\lambda_{2}^{{1}/{3}})(1-\lambda_{2}^{{1}/{12}})},$ we have 
{
\begin{align*}
    R_i(t)\hspace{-2pt}<\hspace{-2pt}\sum_{\Delta_k>0}\hspace{-2pt}\Delta_k\hspace{-2pt}\left(\max\left\{\frac{2N}{(\frac{1}{2}\Delta_k-\alpha_1)^2}\log T ,L,(3M+1)N\right\}\hspace{-2pt}+\hspace{-2pt}\alpha_2\right).
\end{align*}}
\end{proof}

We have Remark \ref{rmk:order_clean} for the order of regret $R_i(T)$.
\begin{Remark}\label{rmk:order_clean}
There are two important terms affecting the order of regret: $\frac{2N}{(\frac{1}{2}\Delta_k-\alpha_1)^2}\log T$ and $L$.
The order of the former term is $O(N \log  T)$, and the order of the latter term does not depend on $T$  since $L$ is determined by $\lambda_2$ and $N$.
Recall that $L$ is the value which makes ${\lambda_{2}^{{t}/{6}}}/{(1-\lambda_{2}^{{1}/{3}})}<{(N t)^{-1}}$ hold for all $t\geq L.$
Let $t= 6 \gamma \log_{\lambda_2^{-1}} N$. The inequality ${\lambda_{2}^{{t}/{6}}}/{(1-\lambda_{2}^{{1}/{3}})}<{(N t)^{-1}}$ becomes:
\[
\frac{6N}{1-\lambda_2^{1/3}} \frac{1}{N^\gamma} < \frac{1}{\log_{\lambda_2^{-1}} N^\gamma}.
\]
There always exists a positive $\gamma$ such that the above inequality holds.
Thus the order of $R_i(T)$ is 
$O(\max\{N M \log T, M\log_{\lambda_2^{-1}} N\})$.
\end{Remark}

As for the lower bound of the regret, consider a trivial case where the graph $G$ is fully connected. Then easily we can show the regret of our algorithm will be lower bounded by an ideal setting where all agents will receive the reward information from everyone else simultaneously with no delay. This setting reduces to a centralized bandit setting and by calling the classical results we know the regret is lower bounded by $\Omega(\log T)$. It remains a challenging and interesting question to understand the tightness of our bound in terms of the number of agents $N$ and the graph~$G$.

\section{\texttt{Fed\_UCB}: Privacy Preserving {\texttt{Gossip\_UCB}}}

{
Prior literature has shown that directly leaking some information that might appear to be ``anonymized'' can in fact be used to cross-reference with other datasets to breach privacy \cite{sweeney2000simple}. In this section, we seek for a federated bandit solution with worst-case privacy guarantees (even with arbitrarily powerful adversary).}
When guaranteeing an $\epsilon$-differential privacy in a one-shot setting, adding Laplacian noise $\gamma\sim {\sf Lap}(\frac{1}{\epsilon})$ to the observation often suffices, where a larger $\epsilon$ indicate a lower privacy level.
However, preserving privacy in a sequential setting is in general hard due to the continual and sequential revelation of observations. 
That is, in addition to preserving the privacy of $X_{i,k}(t)$, we also need to protect it in each $\tau=t,t+1,...,T$ steps. %

\subsection{Algorithm}
To preserve at least $\epsilon$-DP in $T$ time slots, a naive extension of the Laplace mechanism \cite{dwork2006calibrating} is adding Laplacian noise ${\sf Lap}(\frac{T}{\epsilon})$ to each observation $X_{i,k}(t)$. The noise introduced in each time step grows linearly w.r.t. $T$, i.e. $O(\frac{T}{\epsilon})$. %
To add a mild noise and maintain the same privacy level at the same time, we apply the partial sums idea \cite{chan2011private} to $\tilde{X}_{i,k}(t)$. 
Since both the gossiping information $\theta_{i,k}(t)$ and the selection information $n_{i,k}(t)$ are functions of ${X}_{i,k}(t)$, this approach also preserves the privacy for \texttt{Gossip\_UCB} by data processing inequality. 

\begin{algorithm}[!t]
\LinesNumbered
\DontPrintSemicolon   
\SetKw{KwRet}{Return}
\KwIn{Observations $\{X(\tau)\}_{\tau=1}^{t}$, $t' = t$, $p=0$}
\KwOut{$P$ partial sums: $\{ \widehat X_1^{\sf ps}, \widehat X_2^{\sf ps}, \cdots, \widehat X_P^{\sf ps} \}$}
\BlankLine
\While{$t'\ne 0$}{
$d = \max \{i: {\sf bin}(t')[i] = 1\}$ {\algcom{\hfill // convert $t'$ as a binary string and find the rightmost digit that is $1$} }\\
${\sf bin}(t')[d] = 0, q(t') = {\sf int}({\sf bin}(t'))$ \algcom {\hfill // flip that digit to $0$, convert is back to decimal and get $q(t')$}\\
$\widehat X_p^{\sf ps} = \sum_{\tau=q(t')+1}^{t'} \widehat X_{i,k}(\tau)$ \hfill \algcom {// get a partial sum}\\
$t' = q(t'), p = p + 1$\;
}
\caption{Create partial sums}
\label{alg:partial_sum}
\end{algorithm}

\begin{algorithm}[!t]
\LinesNumbered
\DontPrintSemicolon   
\SetKw{KwRet}{Return}
\KwIn{Observations $ \{X(\tau)\}_{\tau=1}^{t}$, mapping from partial sums to noise: $\mathcal P$, privacy level $\epsilon'$}

Get a list with $P$ partial sums: $\{ \widehat X_1^{\sf ps}, \widehat X_2^{\sf ps}, \cdots, \widehat X_P^{\sf ps} \}$ using Algorithm~\ref{alg:partial_sum}, $\tilde X = 0$ \;
\For{$p = 1,\ldots,P$}{
$\mathds 1_{\sf p} = \mathds 1( \widehat X_p^{\sf ps} \text{ is defined in } \mathcal P )$ \hfill \algcom {// set as 1 when Laplacian noise is already added to $\widehat X_p^{\sf ps}$} \\
\hspace{0pt} \algcom{/* \hfill  add i.i.d Laplacian noise to partial sum $\widehat X_p^{\sf ps}$ if it appears for the first time \hfill */} \\
$\tilde X \leftarrow \tilde X + \widehat X_p^{\sf ps} + \mathds 1_{\sf p} \mathcal P(\widehat X_p^{\sf ps}) + (1-\mathds 1_{\sf p}) {\sf Lap}(1/\epsilon')$ 
}
\KwRet{$\tilde X/n_{i,k}(t)$}
\caption{ Add Laplacian Noise to Partial Sums}
\label{alg:privacy}
\end{algorithm}

\begin{algorithm}[!ht]
\LinesNumbered
\DontPrintSemicolon   
\SetKw{KwRet}{Return}
\KwIn{$ G,T,C_{i,k}(t)$}
\BlankLine
\textbf{Initialization:} Each agent pulls each arm exactly once, and receives a reward $X_{i,k}(0)$, $i\in [N]$, $j\in[M]$. Set $n_{i,k}(0)=1$, $\tilde\vartheta_{i,k}(0)=\tilde{X}_{i,k}(0)=X_{i,k}(0).$
\BlankLine
\For{$t = 1,\ldots,T$}{
$\mathcal A_i = \varnothing$ \;
$n_{i,k}(t):=n_{i,k}(t-1), \forall k \in [M]$  \;
$\tilde{n}_{i,k}(t+1)=\max\{n_{i,k}(t),\tilde{n}_{j,k}(t),j\in \scr N_i\},\forall k\in[M]$\;
Put $k$ into set $\scr{A}_{i}$ ~\textbf{if} ~$n_{i,k}(t)<\tilde{n}_{i,k}(t)-N, \forall k \in [M]$   \algcom{\hfill // local consistency requirements} \\
\eIf{$\scr{A}_{i}$ is empty}{
\For{$k=1,\dots,M$}{
$\tilde Q_{i,k}(t) :=\tilde\vartheta_{i,k}(t-1)+\tilde C_{i,k}(t)$ \algcom{\hfill // update the belief on each arm, $C_{i,k}(t)$ is chosen as (\ref{eq:confidence_fed})} \\
$a_{i}(t)=\arg \max_{k} \tilde Q_{i,k}(t)$ \algcom{\hfill // select the best arm to pull} 
}}{
$a_{i}(t)$ is randomly selected from $\scr A_{i}$}
\algcom{/* \hfill Lines \ref{line:dp1}, \ref{line:dp2}: Get DP-observations with input $\{\widehat X_{i,k}(\tau)\}_{\tau=1}^t$, $\mathcal P$, and $\epsilon$ \hfill */} \\
Get noisy observation $\tilde X_{i,a_i(t)}(t)$ following Algorithm \ref{alg:privacy}, update $\mathcal P$ \label{line:dp1} \;
$\tilde X_{i,k}(t) = \tilde X_{i,k}(t-1), \forall k \ne a_i(t)$ \label{line:dp2}\; 
\eIf{agent $i$ is selected to gossip with agent $j$}{
agent $i$ sends $\tilde\vartheta_{i,k}(t-1)$ to agent $j,$ at the same time, receives $\tilde\vartheta_{j,k}(t-1)$ from agent $j$\;
$\tilde\vartheta_{i,k}(t):=\frac{\tilde\vartheta_{i,k}(t-1)+\tilde\vartheta_{j,k}(t-1)}{2}+\tilde{X}_{i,k}(t)-\tilde{X}_{i,k}(t-1)$ \algcom{\hfill // gossiping update} 
}{
$\tilde\vartheta_{i,k}(t):=\tilde\vartheta_{i,k}(t-1)+\tilde{X}_{i,k}(t)-\tilde{X}_{i,k}(t-1)$ \algcom{\hfill // normal update} } 
}
\caption{\texttt{Fed\_UCB}}
\label{alg:Fed_UCB}
\end{algorithm}
Denote by $\widehat X_{i,k}(\tau):= \mathds 1(a_i(\tau) = k)  X_{i,a_i(\tau)}(\tau).$
The partial sum is constructed as in Algorithm~\ref{alg:partial_sum}.
As a result, $\tilde X_{i,k}(t)$ can be written as the summation of no more than $\lceil \log t \rceil$ partial sums:
{  \begin{align*}
&\tilde X_{i,k}(t) = \frac{1}{n_{i,k}(t)}\bigg(\sum_{\tau=q(t)+1}^t \widehat X_{i,k}(\tau) +\sum_{\tau=q(q(t))+1}^{q(t)}\widehat X_{i,k}(\tau)+...+\sum_{\tau=1}^{q(\cdots(q(t)))} \widehat X_{i,k}(\tau)\bigg)~.\label{partialsum:theta}
 \end{align*}}

Independent Laplacian noise $\gamma\sim {\sf Lap}(1/\epsilon')$ is added to each partial sum if there exists observations in that partial sum.
Thus the total privacy guarantee is given by $\lceil \log t \rceil\epsilon$.
Set $\epsilon' := \epsilon \frac{1}{\lceil \log T \rceil}$, where $\epsilon$ is a pre-set privacy level we want to achieve. Then we will achieve at least a total $ \epsilon \frac{1}{\lceil \log T \rceil}\lceil \log t \rceil \le \epsilon$ differential privacy. 
The algorithm for adding Laplacian noise to partial sums is shown in Algorithm~\ref{alg:privacy}.
To make the algorithm preserve $\epsilon$-DP, we need to modify the observation procedure in line 14 of Algorithm~\ref{alg:gossip_UCB}.
Specifically, instead of directly getting $X_{i,a_i(t)}(t)$, Algorithm~\ref{alg:privacy} is implemented with observations $\{\widehat X_{i,k}(\tau)\}_{\tau=1}^t$, the mapping $\mathcal P$ recording the noise added to the previous partial sums, and the privacy level $\epsilon$.
Besides, the confidence bound is set following (\ref{eq:confidence_fed}).
Note \texttt{Gossip\_UCB} is a special case of \texttt{Fed\_UCB} when $\epsilon$ is infinity.
We summarize \texttt{Fed\_UCB} in Algorithm~\ref{alg:Fed_UCB}.
Note the time complexity of line \ref{line:dp1} is $O(\log t)$.

\subsection{Concentration Bound for Local Privacy-Preserving Noisy Estimates}
Compared with \texttt{Gossip\_UCB}, the variance proxy for $\tilde \vartheta_{i,k}(t)$ consists of two parts: the variance proxy for observations and the variance proxy for Laplacian noise. The former one is the same as that in \texttt{Gossip\_UCB}. We bound the latter one in Lemma~\ref{varianceproxy_feducb}.
\begin{Lemma}\label{varianceproxy_feducb}
When $n_{i,k}(t)\geq \max\{L,(3M+1)N\}, \forall i\in[N]$, with probability at least $1-p_1,$ the concentration bound for $\tilde \vartheta_{i,k}(t)$ is $
     \mathbb P(|\tilde\vartheta_{i,k}(t) - \mu_{k}|\geq \tilde C_{i,k}(t))<\frac{2}{t^2},$
where $p_1=\frac{2N}{n_{i,k}(t)}+p_0$ and
\begin{equation}\label{eq:confidence_fed}
    \tilde C_{i,k}(t)= \alpha_1 + \sqrt{2N\left(\frac{128N \log^2T \cdot \log t \cdot \log n_{i,k}(t)}{n^2_{i,k}(t)\epsilon^2}+\frac{1}{n_{i,k}(t)}\right)\log t}.
\end{equation}
Note $\epsilon$ is the privacy level and parameters $L, p_0, \alpha_1$ are the same as Lemma~\ref{concentration1}.
\end{Lemma}

\begin{proof} (Sketch)
Recall $\epsilon' := \epsilon \frac{1}{\lceil \log T \rceil}$.
Up to each time $t$, there are at most $\lceil \log t \rceil $ noise terms added to each $\tilde{X}_{i,k}(t)$.
Let $\gamma_{\tau}\sim {\sf Lap}(1/\epsilon')$ and $\Gamma:= \sum_{\tau=1}^{\lceil \log t \rceil } \gamma_{\tau}$.
Set $$\lambda:= \frac{4   \log T  \cdot \sqrt{\log t \cdot \log n_{i,k}(t)}}{\epsilon}.$$
The sum of these Laplace noise terms satisfy the following inequality \cite{dwork2006calibrating}:  
\[
\mathbb P(\Gamma \geq \lambda) \leq \exp\left(-\frac{\lambda^2}{8\sum_{\tau} \frac{1}{\epsilon'^2}}\right) \le  n_{i,k}^{-2}(t).
\]
Thus with probability at most ${n_{i,k}^{-2}(t)}$, the noise term added to each empirical mean $\tilde{X}$ is larger than $\frac{\lambda }{n_{i,k}(t)}$. 
Now we want to figure out how the added noise would affect the regret, more specifically, the $\vartheta_{i,k}(t).$ According to \eqref{tilde}, $\vartheta_{i,k}(t)$ is a linear combination of $\tilde{X}_{i,k}(t),$ denote $\Delta \vartheta_{i,k}(t)$ as the total noise added in $\tilde{\vartheta}_{i,k}(t),$  $S_{i,k}(t)$ as the cumulative noise added in $\sum_{\tau=1}^t \tilde{\theta}(\tau).$ Then $\Delta \vartheta_{i,k}(t)$ is a linear combination of $S_{i,k}(t)$ and we have with probability at least $1-p_1,$
\begin{align*}
    |\Delta \vartheta_{i,k}(t)|\leq\frac{8N}{n_{i,k}(t)}\log T  \cdot \sqrt{\log t \cdot \log n_{i,k}(t)}/\epsilon,
\end{align*}
where $$p_1=\frac{2N}{n_{i,k}(t)}+p_0.$$
Define $\tilde \sigma^2_{i,k}(t)$ as the variance proxy of $\tilde{\vartheta}_{i,k}(t).$
According to Algorithm~\ref{alg:privacy}, each new observation comes with an additional independent Laplacian noise term. Noting the linear combination derived from gossiping procedures will not change the independence, the gossiped noise $\Delta \vartheta_{i,k}(t)$ is independent of each $X_{j,k}(\tau),$ for all  $j\in [N], \tau\in\{0,\ldots,t\}$.
According to the definition of the variance proxy, we know $\sigma^2_{\Delta} = |\Delta \vartheta_{i,k}(t)|^2$. Summing two variance proxies, i.e. $\sigma^2_{i,k}(t)+\sigma^2_{\Delta}$ returns $\tilde \sigma^2_{i,k}(t)$.
Following the corresponding procedures as Section~\ref{sec:concentration_bound}, we can prove this lemma. 
\end{proof}

\subsection{Regret Upper Bound for \texttt{Fed\_UCB}}
With the existence of Laplacian noise, Theorem~\ref{thm:main} can be extended for \texttt{Fed\_UCB} as follows.
\begin{Theorem}\label{privacy_theorem}
{({Regret upper bound for \texttt{Fed\_UCB}})}
For the $\epsilon$-differentially private \texttt{Fed\_UCB} algorithm with bounded reward over $[0,1]$ and confidence bound $\tilde C_{i,k}(t)$, 
the regret of each agent $i$ until time $T$ satisfies the bound 
{\small
\begin{equation*}%
    R_i(T)<\sum_{\Delta_k>0}\Delta_k\left(\max\left\{\frac{N  \log T \cdot \left( 1+ \sqrt{1+ \left({16(\frac{\Delta_k}{2}-\alpha_1)}/{\epsilon}\right)^2 \log^3 T} \right)}{(\frac{\Delta_k}{2}-\alpha_1)^2} 
     ,L,(3M+1)N\right\}+4N\log T+\alpha_3\right),
\end{equation*}}
where $\alpha_3:= (3M-1)N + \frac{2\pi^{2}}{3}+\frac{2\lambda_2^{1/12}}{(1-\lambda_{2}^{{1}/{3}})(1-\lambda_{2}^{{1}/{12}})} + 4N.$
\end{Theorem}

\begin{proof}

Following the concentration bound for $\tilde\vartheta_{i,k}(t)$, we set $\tilde C_{i,k}(t)$ in (\ref{eq:confidence_fed}) as the upper confidence bound.
Notice %
$$\tilde C_{i,k}(t)<\sqrt{\left(\frac{128N\log^4T}{n_{i,k}(t)\epsilon^2}+1\right)\cdot\frac{2N \log t}{n_{i,k}(t)}}+\alpha_1,$$
following the similar steps in the analyses of \texttt{Gossip\_UCB}, we get with probability at least $1-p_1,$ after time $t'$ which satisfies:
\begin{align*}
    n_{i,k}(t')=\frac{N \log T \cdot \left( 1+ \sqrt{1+ \left(16(\frac{\Delta_k}{2}-\alpha_1)/\epsilon\right)^2 \log^3 T} \right)}{(\frac{\Delta_k}{2}-\alpha_1)^2},
\end{align*}
the expected error resulting from agent $i$ selecting arm $k$ is bounded by $(3M-1)N+\frac{2\pi^{2}}{3}.$ 
The remaining part (similar to Case 2B and Case 3B in Section~\ref{eq:secRegretGossip}) is bounded by $\frac{2\lambda_2^{1/12}}{(1-\lambda_{2}^{{1}/{3}})(1-\lambda_{2}^{{1}/{12}})} + 4N(\log T+1)$.
Define $\alpha_3 = \alpha_2 + 4N$.
Therefore, the expected number of non-optimal arms is bounded as 
\begin{equation*} %
    \mathbf{E}[n_{i,k}(T)]<\max\left\{ \frac{N \log T \cdot \left( 1+ \sqrt{1+ \left(16(\frac{\Delta_k}{2}-\alpha_1)/\epsilon\right)^2 \log^3 T} \right)}{(\frac{\Delta_k}{2}-\alpha_1)^2} 
     ,L,(3M+1)N\right\}+ 4N\log T+ \alpha_3.
\end{equation*}

Using the union bound, the regret of agent $i$ until time $T$ is bounded as:
\begin{equation*}%
\begin{split}
    R_i(t)
    <&\sum_{\Delta_k>0}\Delta_k\Bigg(\max\Bigg\{\frac{N \log T \cdot \bigg( 1+ \sqrt{1+ \big({16(\frac{\Delta_k}{2}-\alpha_1)}/{\epsilon}\big)^2 \log^3 T} \bigg)}{(\frac{\Delta_k}{2}-\alpha_1)^2} ,L,(3M+1)N\Bigg\}+4N \log T +\alpha_3 \Bigg).
     \end{split}
\end{equation*}

\end{proof}

\begin{Remark}\label{rmk:orderDP}
The order of $R_i(T)$ is $O(\max \{\frac{N M}{\epsilon} \log^{2.5} T , M(N\log T +\log_{\lambda_2^{-1}} N) \})$.
\end{Remark}

Remark~\ref{rmk:orderDP} shows introducing $\epsilon$-DP leads to a regret bound in the order $O(\log^{2.5}(T))$, which is the same (in terms of $T$) as the one reported for the centralized single-agent algorithm \cite{Mishra2014PrivateSM}.\footnote{Note the regret bound may be different due to the adoption of different privacy algorithms, and the coupling effects caused by the delayed transmission of heterogeneous reward information, e.g., \cite{tossou2016algorithms} proposes an single-agent algorithm with regret order $O(\epsilon^{-1}+\log T)$ while achieving the $(\epsilon,\delta)$-DP, $\delta>0$. Their algorithm achieves a worse guaranteed privacy level than ours when $\delta=0.$}

\section{Experiments}
\begin{figure*}[!t]
     \centering
     \begin{subfigure}[b]{0.326\textwidth}
         \centering
         \includegraphics[width=\textwidth]{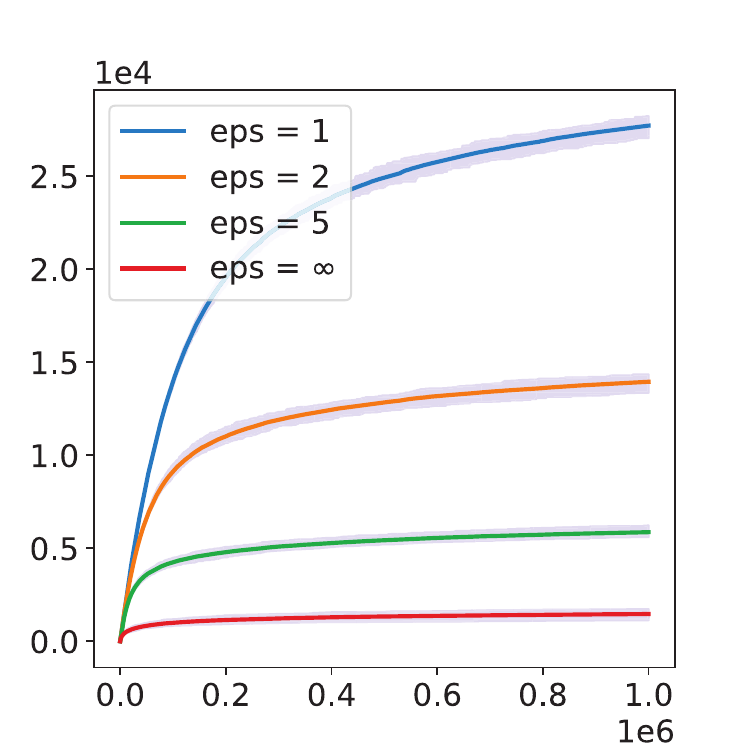}
         \caption{}%
         \label{fig:synthetic1}
     \end{subfigure}
     \hfill
     \begin{subfigure}[b]{0.326\textwidth}
         \centering
         \includegraphics[width=\textwidth]{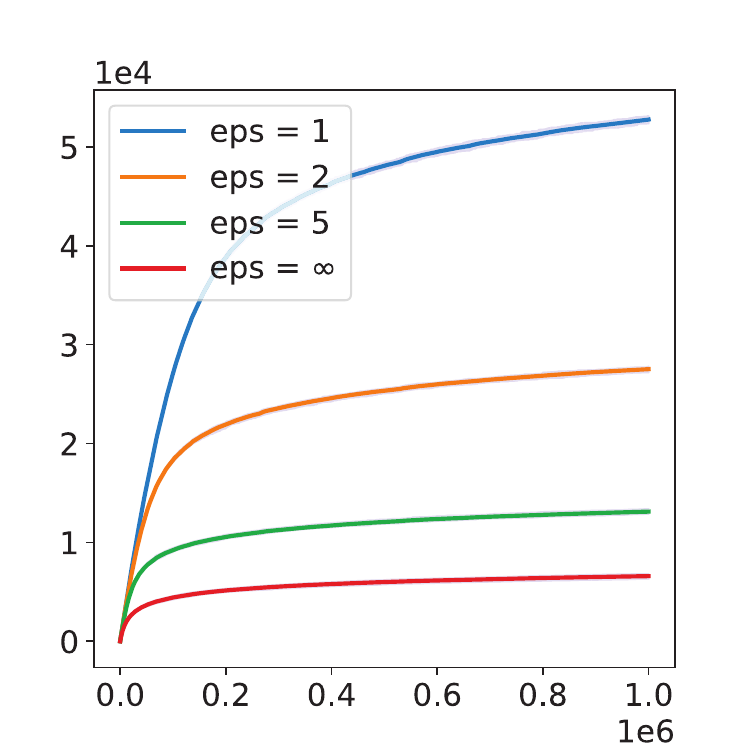}
         \caption{}
         \label{fig:synthetic2}
     \end{subfigure}
     \hfill
     \begin{subfigure}[b]{0.326\textwidth}
         \centering
         \includegraphics[width=\textwidth]{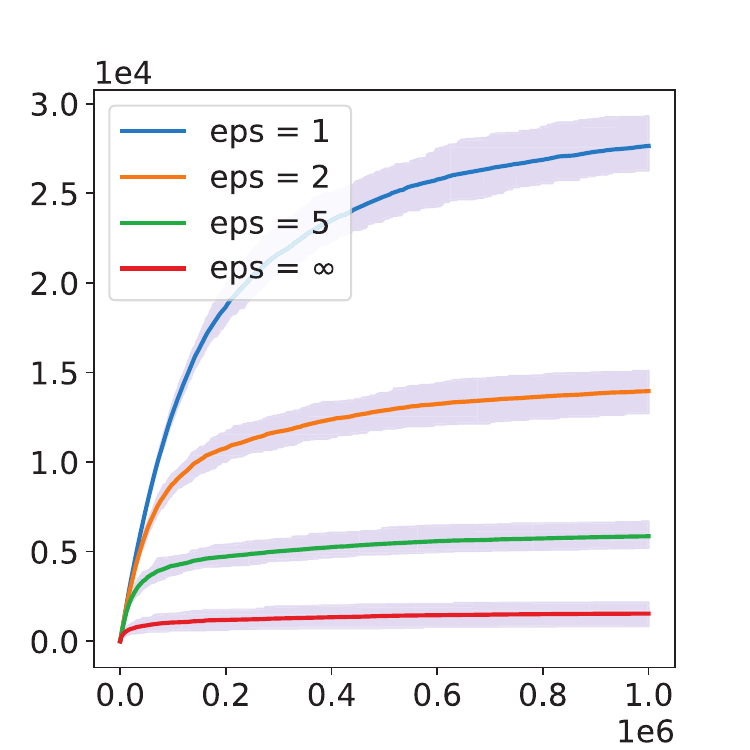}
         \caption{}
         \label{fig:real}
     \end{subfigure}
        \caption{Regret of \texttt{Fed\_UCB} over time when the algorithm preserves $\epsilon$-DP. Higher $\epsilon$ indicates lower privacy level. 
        (a) Synthetic data with $N=3, M=5, \lambda_2 = 0.5$; (b) Synthetic data with $N=10, M=10, \lambda_2 \approx 0.99$; (c) Real hospital data with $N=3,M=5,\lambda_2 = 0.5$.}
        \label{fig:simu}
\noindent\rule[0.25\baselineskip]{\textwidth}{.5pt}
\end{figure*}

In addition to the theoretical guarantees, we test the performance of \texttt{Fed\_UCB} in multi-agent stochastic multi-arm bandits problems on both synthetic datasets and real datasets.
Each experiment is run for $100$ times with i.i.d. noise and the regret is averaged over $100$ trials and $N$ agents. 
Shadow areas indicate the range of realized regret (minimum and maximum in $100$ trials).
When $\epsilon=\infty$, there is no extra privacy preservation step added as discussed in \texttt{Gossip\_UCB}. 
To show the effect of $\epsilon$-DP preservation on the regret of \texttt{Fed\_UCB}, we add Laplacian noise following Algorithm \ref{alg:privacy}. The total privacy levels are set as $\epsilon = 1,2,5$, respectively.

\subsection{Synthetic datasets}
In the synthetic data, the expected belief of agent $i$ toward arm $k$, i.e. $\mu_{i,k}$, is generated randomly in $[0,1]$. The rank of arms from each agent's perspective is different. Zero-mean Gaussian noise with variance $1$ is added to each observation. 
Fig.~\ref{fig:simu}(a) and Fig.~\ref{fig:simu}(b) show the simulation results in different settings. When the increase of regret is stable, e.g. $t=$6e5, in both figures, the ratios of regret with $\epsilon_i\in\{1,2,5\}$ are approximately $1:1/2:1/5 = 1/\epsilon_1 : 1/\epsilon_2 : 1/\epsilon_3$, which is consistent with our bound on the regret of \texttt{Fed\_UCB}.

\subsection{Real-world datasets}

We use the UCI diabetes dataset \cite{strack2014impact}, which is a medical dataset including $101,766$ inpatient records in US hospitals. There are over $50$ features representing patient and hospital outcomes, e.g., results of some laboratory tests, personal information of patients, and the specific medications administered during the encounter. The label encodes three readmission states of patients: no record, readmitted in more than $30$ days, and readmitted in less than $30$ days.
\paragraph{Construction of arms:}
There are $23$ diabetes medications: %
metformin, repaglinide, nateglinide, chlorpropamide, glimepiride, %
acetohexamide, glipizide, glyburide, tolbutamide, pioglitazone, %
rosiglitazone, acarbose, miglitol, troglitazone, tolazamide, %
examide, sitagliptin, insulin, glyburide-metformin, glipizide-metformin, %
Each medication contains four sub-features, i.e., \emph{No}, \emph{Down}, \emph{Up}, \emph{Steady}, indicating whether the drug was prescribed or there was a change in the dosage. The number of each sub-feature is pretty unbalanced. Generally, there are much less samples with sub-features \emph{Down} and \emph{Up} compared with \emph{No} and \emph{Steady}. To ensure a sufficient number of samples, we only focus on the sub-feature \emph{No} and \emph{Steady}. Additionally, noting the number of \emph{Steady} samples in different medications is also very imbalanced, we combine some medications with similar rewards and get $5$ arms. They are:
\begin{itemize}\setlength\itemsep{0pt}
    \item \emph{Arm $1$}: Use steady dosage of insulin;
    \item \emph{Arm $2$}: Use steady dosage of metformin;
    \item \emph{Arm $3$}: Use steady dosage of repaglinide, glipizide, rosiglitazone, acarbose, miglitol or troglitazone;
    \item \emph{Arm $4$}: Use steady dosage of nateglinide, chlorpropamide, glyburide, examide, glimepiride, acetohexamide, glyburide, tolbutamide, pioglitazone, sitagliptin, glyburide-metformin, glipizide-metformin, glimepiride-pioglitazone, metformin-rosiglitazone, or metformin-pioglitazone;
    \item \emph{Arm $5$}: No action in any medications.
\end{itemize}

\paragraph{Construction of rewards:} There are $3$ readmission states: \emph{no record}, \emph{readmitted in more than $30$ days}, and \emph{readmitted in less than $30$ days}.
If the patient is readmitted in less than $30$ days, we can infer the treatment (medication) does not work. On the contrary, if the patient is readmitted in more than $30$ days, we can infer the treatment (medication) works in some sense.
If there is no record for this patient, it may because the patient chooses other hospitals or is cured, which makes it difficult to determine the effectiveness of the treatment.
Additionally, we cannot simply discard these samples since they account for a large population.
In our experiments, we assume \emph{no record} means the treatment does not work well.
There are several evidences:
\begin{itemize}\setlength\itemsep{0pt}
\item \emph{Evidence 1}: Diabetes are chronic diseases, it is hard to be cured in one hospital stay; 
\item \emph{Evidence 2}: By simple logistic regression, we can find the readmission states are highly-related to the discharge disposition, including transferring other hospitals or medical institutions. This phenomenon supports our speculation that \emph{no record} mainly shows bad treatments; 
\item \emph{Evidence 3}: From our numerical results, assuming \emph{no record} indicating the positive effect of treatments leads to a unreasonable result: nearly all steady medications have negative effects. 
\end{itemize}
Therefore, we map \emph{no record} and \emph{readmitted in less than $30$ days} to reward $0$, and map \emph{readmitted in more than $30$ days} to reward $1$.

 In this experiment, we focus on the relationship between medications and readmission states. Noting no one wants to be known about their readmission status, we consider a federated bandit setting where hospitals (agents) do not directly share their rewards (readmission states) due to privacy concerns. On one hand, the privacy can be protected in some sense the gossiping procedure since the variable for gossiping is a weighted combination of the protected data with complicated (and possibly unknown) weights. On the other hand, the hospital can only use a noisy copy of readmission states in the DP setting.
Each hospital protects their patient's privacy in this mechanism\footnote{To further protect the patient's privacy, the Laplacian noise can be added by some trust-worthy institutions such that hospitals cannot get the clean data.}. 
We map $23$ types of medications into $5$ arms, and randomly divide the samples into $3$ equal-sized set belonging to three hospitals (agents). 
From our numerical results, local biased indeed occurs when the dataset is separated into several parts according to the order in which the samples appear.
We leave detailed data processing procedures in supplementary materials.
The feasibility of federated bandit in a real hospital dataset is testified in Fig.~\ref{fig:simu}(c).

\section{Conclusion}
In this paper, we have proposed \texttt{Gossip\_UCB} for solving a gossiping bandit learning problem, where a network of agents aim to learn to converge to selecting the best arm both locally and globally through gossiping, and its differentially private variant, \texttt{Fed\_UCB}, for preserving $\epsilon$-differential privacy of the agents' local data.
We have shown that both \texttt{Gossip\_UCB} and \texttt{Fed\_UCB} achieve weak regret at an order depending on 
the size of agents, the number of arms, time horizon, and connectivity of the graph. In addition to the theoretical bounds on regret, experiments on both synthetic data and real data also verify the feasibility of the proposed gossiping approach of federated bandit.
Future work may include extending this framework to contextual bandits \cite{dubey2020kernel} with local features or bandits with continuous arms \cite{combes2020unimodal}.

\section*{Acknowledgement}
This work is partially supported by the National Science Foundation (NSF) under grant IIS-2007951 and the Office of Naval Research under grant
N00014-20-1-22.

\bibliographystyle{ACM-Reference-Format}
\bibliography{main}


\appendix

\vspace{10pt}
{\Large \bf Appendix}
\vspace{10pt}

The analyses are based on sub-Gaussian random variables.
We present the following preliminaries.

\section{Preliminaries of sub-Gaussian random variables}\label{subgaussian}

\begin{Definition}
A random variable $X$ with $\mu=\mathbb{E}[X]$ is called $\sigma^{2}$ sub-Gaussian if there is a positive $\sigma$ such that
\begin{align*}
    \mathbb{E}[e^{\lambda(X-\mu)}]\leq e^{\frac{\sigma^{2}\lambda^{2}}{2}}, \;\forall \lambda\in \R,
\end{align*}
where such $\sigma^{2}$ is called a variance proxy, and the smallest variance proxy is called the optimal variance proxy. 
\end{Definition}

\begin{Property} (Inequality)\label{hoeffding}
A sub-Gaussian random variable $X$ satisfies:
\begin{equation*}%
    \begin{split}
        \bbb{P}(X-\mu\geq a)&\leq e^{-\frac{a^{2}}{2\sigma^{2}}},\\
        \bbb{P}(\mu-X\geq a)&\leq e^{-\frac{a^{2}}{2\sigma^{2}}}.
    \end{split}
\end{equation*}
\end{Property}

\begin{Property} (Sufficient condition)\label{sufficientcondition}
If $X$ is a random variable with finite mean $\mu$ and $a\leq X\leq b$ almost surely, then $X$ is $\frac{(b-a)^{2}}{4}$ sub-Gaussian.
\end{Property}

\begin{Property} (Additivity)\label{additivity}
If $X_{1}$ is $\sigma_{1}^{2}$ sub-Gaussian and for $2\leq i \leq n, (X_{i}|X_{1},\ldots, X_{i-1})$ is $\sigma_{i}^{2}$ sub-Gaussian with $\sigma_{i}$ being free of $X_{1},\ldots,X_{i-1},$ then $X_{1}+\cdots+X_{i}$ is sub-Gaussian with $\sigma_{1}^{2}+\cdots+\sigma_{i}^{2}$ being one of its variance proxy.
\end{Property}

\section{Full Proofs}

\subsection{Proof for Lemma~\ref{concentration1}}\label{sup_proxy}
\begin{customlem}{4}\label{sup:varianceproxy}
When $n_{i,k}(t)\geq \max\{L,(3M+1)N\}, \forall i\in[N],$ with probability at least $1-p_0,$ $|\vartheta_{i,k}(t)-\mathbb{E}(\vartheta_{i,k}(t))|$ bounds as follows:
\begin{align*}
    & \mathbb P\left(|\vartheta_{i,k}(t)-\mathbb{E}[\vartheta_{i,k}(t)]|\geq \sqrt{\frac{2N}{n_{i,k}(t)}\log t}\right)  <\frac{2}{t^2},
\end{align*}
where 
$p_0=
\frac{\lambda_{2}^{{n_{i,k}(t)}/{12}}}{1-\lambda_{2}^{{1}/{3}}}
,$ $L$ is the minimal value that satisfies
${\lambda_{2}^{{t}/{6}}}/{(1-\lambda_{2}^{{1}/{3}})}<{(N t)^{-1}}$, $\forall t\geq L.$ 
\end{customlem}

\begin{proof}

We can infer from the algorithm that each $\vartheta_{i,k}(t)$ is a linear combination of $X_{j,k}(\tau)$, for all  $j\in [N], \tau\in\{0,\ldots,t\}.$ Define $c_{i,k,j}^{(\tau)}$ as the corresponding coefficient of such  $X_{j,k}(\tau)$ in $\vartheta_{i,k}(t).$ To find the variance proxy of $\vartheta_{i,k}(t),$ which is $\frac{1}{4}\sum_{j=1}^{N}\sum_{\tau=1}^{t}|c_{i,k,j}^{(\tau)}|^2$ according to Property~\ref{sufficientcondition} and Property~\ref{additivity} in Appendix \ref{subgaussian}, we will estimate the value of $c_{i,k,j}^{(\tau)}$ in the following steps.

The vector form of the iterative process can be expressed as:
\begin{align}
    \vartheta_{k}(t)=W(t)\vartheta_{k}(t-1)+\Tilde{X}_{k}(t)-\Tilde{X}_{k}(t-1),\label{vectorupdate}
\end{align}
for all $k\in [M],$ where $W(t)=I-\frac{1}{2}(e_i-e_j)(e_i-e_j)^{\top}$ is the gossip updating matrix when agent $i$ and $j$ are selected to exchange information at time $t$. 
According to \eqref{vectorupdate}, we have: 
\begin{align*}
    \vartheta_{k}(t)&=W(t)\vartheta_{k}(t-1)+\Tilde{X}_{k}(t)-\Tilde{X}_{k}(t-1)\\
    &=\left(W(t)\cdots W(1)-W(t)\cdots W(2)\right)X_{k}(0) +(W(t)\cdots W(2)-W(t)\cdots W(3))\Tilde{X}_{k}(1)\\
    &\quad +\cdots+(W(t)-I)\tilde{X}_{k}(t-1)+\tilde{X}_{k}(t).
\end{align*}
Define $\tau_{1},\ldots, \tau_{n_{j,k}(t)}$ as the time instance before $t$ when agent $j$ pulls arm $k$, so $\tau_{1}=0$, and we have
\begin{equation}
    \begin{split}
        \vartheta_{i,k}(t)&=\sum_{j}\bigg([W(t)\cdots W(\tau_{1}+1)-W(t)\cdots W(\tau_{2}+1)]_{i,j}\Tilde{X}_{j,k}(\tau_{1})\\&\quad+\cdots+[W(t)\cdots W(\tau_{n_{j,k}(t)-1}+1)-W(t)\cdots W(\tau_{n_{j,k}(t)}+1)]_{i,j}\Tilde{X}_{j,k}(\tau_{n_{j,k}(t)-1})\\&\quad +[W(t)\cdots W(\tau_{n_{j,k}(t)}+1)]_{i,j}\Tilde{X}_{j,k}(\tau_{n_{j,k}(t)})\bigg),\label{tilde}
    \end{split}
\end{equation}
and
\begin{equation}
    \begin{split}
        c_{i,k,j}^{(0)}&=\bigg[(W(t)\cdots W(\tau_{1}+1)-W(t)\cdots W(\tau_{2}+1))+\cdots\\&\quad+\frac{W(t)\cdots W(\tau_{n_{j,k}(t)-1}+1)-W(t)\cdots W(\tau_{n_{j,k}(t)}+1)}{n_{j,k}(t)-1}+\frac{W(t)\cdots W(\tau_{n_{j,k}(t)}+1)}{n_{j,k}(t)}\bigg]_{i,j}, \label{c0}
    \end{split}
\end{equation}
where $[\cdot]_{i,j}$ denotes the entry in the $i$-th row and the $j$-th column.
Notice due to the delay in updating, there exists some $h',$ when $h>h', \tilde{X}_{j,k}(\tau_{h})$ can never be transmitted to agent $i$ till time $t,$ in this case, the coefficient of such $\tilde{X}_{j,k}(\tau_{h})$ in \eqref{tilde} is zero. The number of non-zero terms in \eqref{c0} is always no larger than $\sum_{j=1}^{N}n_{j,k}(t).$ 

We can also write \eqref{c0} as 
 \begin{equation}
     \begin{split}
         c_{i,k,j}^{(0)}=\left[W(t)\cdots W(\tau_{1}+1)-\sum_{h=2}^{n_{j,k}(t)}\frac{W(t)\cdots W(\tau_{h}+1)}{(h-1)h}\right]_{i,j}.
     \end{split}
 \end{equation}
 
 We want to estimate the value of $W(t)\cdots W(\tau_{h}+1).$ According to \cite{Boyd2006}, we have
 \begin{align}
     \mathbb P\left(\left|[W(t)\cdots W(\tau_{h}+1)]_{i,j}-\frac{1}{N}\right|>\lambda_{2}^{\frac{1}{3}(t-\tau_{h})}\right)<\lambda_{2}^{\frac{1}{3}(t-\tau_{h})},\label{Boyd}
 \end{align}
 where $\lambda_2 \in (0,1)$ is the second largest eigenvalue of the expected gossip matrix $W$. Then with probability at least $1-\sum_{h=\frac{n_{j,k}(t)}{2}}^{n_{j,k}(t)}\lambda_2^{\frac{1}{3}h}$, we have
 \begin{align*}
c_{i,k,j}^{(0)} \leq & \left[W(t)\cdots W(\tau_{1}+1)-\sum_{h=2}^{\frac{n_{j,k}(t)}{2}}\frac{W(t)\cdots W(\tau_{h})}{(h-1)h)}\right]_{i,j}\\
<&\frac{1}{N}\left(1-\sum_{h=2}^{\frac{n_{j,k}(t)}{2}}\frac{1}{(h-1)h}\right)+\sum_{h=\frac{n_{j,k}(t)}{2}}^{n_{j,k}(t)}\lambda_2^{\frac{1}{3}h}<\frac{1}{N}\frac{2}{n_{j,k}(t)}+\frac{\lambda_{2}^{\frac{n_{j,k}(t)}{6}}}{1-\lambda_{2}^{\frac{1}{3}}}.
\end{align*}
Due to the global consistency claimed in Lemma~\ref{relation2}, we have
$$
p_0:=\frac{\lambda_{2}^{\frac{n_{i,k}(t)}{12}}}{1-\lambda_{2}^{\frac{1}{3}}} \ge \frac{\lambda_{2}^{\frac{n_{j,k}(t)}{6}}}{1-\lambda_{2}^{\frac{1}{3}}} = 
\sum_{h=\frac{n_{j,k}(t)}{2}}^{n_{j,k}(t)}\lambda_2^{\frac{1}{3}h}.
$$
Thus the above inequality also holds with probability at least $1-p_0$.
According to \cite{Boyd2006}, the event in (\ref{Boyd}) could happen at the same time for all $j$ with probability at least $1-p_0$.
Similarly, we have the lower bound with probability at least $1-p_0$:
 \begin{align*}
     c_{i,k,j}^{(0)} = &\Bigg[W(t)\cdots W(\tau_{1}+1) -\left(\sum_{h=2}^{\frac{n_{j,k}(t)}{2}}+\sum_{h=\frac{n_{j,k}(t)}{2}+1}^{n_{j,k}(t)}\right)\frac{W(t)\cdots W(\tau_{h}+1)}{(h-1)h}\Bigg]_{i,j}\\
     >&\frac{1}{N}-\lambda_2^{\frac{1}{3}t}-\left[\sum_{h=2}^{\frac{n_{j,k}(t)}{2}}W(t)\cdots W(\tau_{h}+1)\right]_{i,j}-\sum_{h=\frac{n_{j,k}(t)}{2}+1}^{n_{j,k}(t)}\frac{1}{(h-1)h}\\
     =&\frac{1}{N}\left(1-\sum_{h=2}^{\frac{n_{j,k}(t)}{2}}\frac{1}{(h-1)h}\right)-\sum_{h=1}^{\frac{n_{j,k}(t)}{2}}\lambda_2^{\frac{1}{3}(t-\tau_h)} -\sum_{h=\frac{n_{j,k}(t)}{2}+1}^{n_{j,k}(t)}\frac{1}{(h-1)h}\\
     >&\frac{1}{N}\left(1-\sum_{h=2}^{\frac{n_{j,k}(t)}{2}}\frac{1}{(h-1)h}\right)-\sum_{h=\frac{n_{j,k}(t)}{2}}^{n_{j,k}(t)}\lambda_2^{\frac{1}{3}h}-\sum_{h=\frac{n_{j,k}(t)}{2}+1}^{n_{j,k}(t)}\frac{1}{(h-1)h}\\
     >&\frac{1}{N}\frac{2}{n_{j,k}(t)}-\frac{\lambda_{2}^{\frac{n_{j,k}(t)}{6}}}{1-\lambda_{2}^{\frac{1}{3}}}-\frac{1}{n_{j,k}(t)}.\\
 \end{align*}

 Thus, with probability at least $1-p_0$:
 \begin{align*}
     |c_{i,k,j}^{(0)}|<\max\left(\left(1-\frac{2}{N}\right),\frac{2}{N}\right)\frac{1}{n_{j,k}(t)}+\frac{\lambda_{2}^{\frac{n_{j,k}(t)}{6}}}{1-\lambda_{2}^{\frac{1}{3}}}.
 \end{align*}
 Let $L$ be the smallest value that makes $\frac{\lambda_{2}^{\frac{t}{6}}}{1-\lambda_{2}^{\frac{1}{3}}}<\frac{1}{N t}$ hold for all $t\geq L.$ Since $N\ge 3,$ when $n_{j,k}(t)\geq L,$ we have $$|c_{i,k,j}^{(0)}|<\frac{1}{n_{j,k}(t)}.$$ 
 It is easy to see other $c_{i,k,j}^{(\tau_{h})}$ is the last $n_{j,k}(t)-h+1$ terms in $c_{i,k,j}^{(0)}$ which is defined in \eqref{c0}, so similarly, $c_{i,k,j}^{(\tau_{h})}$ also have above property, for all $h\leq n_{j,k}(t)$. Thus when $n_{i,k}(t)\geq \max\{L,(3M+1)N\}, \forall i$, using Lemma \ref{relation2}, we know, with probability at least $1-p_0,$
 \begin{align*}
     \sum_{j=1}^{N}\sum_{h=1}^{n_{j,k}(t)}|c_{i,k,j}^{(\tau_{h})}|^2&<\sum_{j=1}^{N}\frac{n_{j,k}(t)}{n^2_{j,k}(t)}<\frac{2N}{n_{i,k}(t)}. 
 \end{align*}
And further we can get  
\begin{align*}
    \sigma^2_{i,k}(t)&=\frac{1}{4}\sum_{j=1}^{N}\sum_{h=1}^{n_{j,k}(t)}|c_{i,k,j}^{(\tau_{h})}|^2 <\frac{N}{2n_{i,k}(t)}.
\end{align*}
According to the Hoeffding's inequality for sub-Gaussian random variables (specified in Property \ref{hoeffding} of Appendix \ref{subgaussian}),
the bound for $|\vartheta_{i,k}(t)-\mathbb{E}(\vartheta_{i,k}(t))|$ is
(with probability at least $1-p_0$)
\begin{align*}
    & \mathbb P\left(|\vartheta_{i,k}(t)-\mathbb{E}[\vartheta_{i,k}(t)]|\geq \sqrt{\frac{2N}{n_{i,k}(t)}\log t}\right)     \leq  2\exp{\frac{\frac{-2N}{n_{i,k}(t)}\log t}{2\sigma^2_{i,k}(t)}}<\frac{2}{t^2}.
\end{align*}
\end{proof}

\subsection{Proof for Lemma~\ref{lem:bias}}
\begin{customlem}{5}\label{sup:bias}
With probability $1$,
\[
|\mathbb{E}(\vartheta_{i,k}(t))-\mu_{k}| \le \sqrt{N}\lambda_2^t, \forall i\in [N], k\in [M].
\]
\end{customlem}
\begin{proof}
Taking expectation on both sides of (\ref{vectorupdate}) yields: 
\begin{align*}
    \mathbb{E}(\vartheta_{k}(t))&=W\mathbb{E}(\vartheta_{i,k}(t-1))+\mathbb{E}(\Tilde{X}_{k}(t)-\Tilde{X}_{k}(t-1))\\
    &=W\mathbb{E}(\vartheta_{i,k}(t-1))\\
    &=W^t\mathbb{E}(\bm \vartheta_{k}(0))
\end{align*}

Since %
\begin{align*}
    \|W^t-\frac{1}{N}\1\cdot \1^{\top}\|_2=\lambda_2^t,
\end{align*}
denoting $\1_{i}$ by the $N\times 1$ column vector with the $i$-th element being $1$ and others being $0$, we have 
\begin{equation}\label{eq:bound_mu}
    \begin{split}
        & |\mathbb{E}(\vartheta_{i,k}(t))-\mu_{k}|\\
  = & |\1_{i}^\top (\mathbb{E}(\bm \vartheta_{k}(t))- \frac{1}{N} \1 \cdot \1^\top \cdot \mathbb{E}(\bm \vartheta_{k}(0))|\\
  = & |\1_{i}^\top (W^t-\frac{1}{N}\1\cdot \1^{\top}) \mathbb{E}(\bm \vartheta_{k}(0))|\\
  \le & \|\1_{i}^\top (W^t-\frac{1}{N}\1\cdot \1^{\top}) \|_2 \|\mathbb{E}(\bm \vartheta_{k}(0))\|_{2}  \qquad\text{(H\"{o}lder's inequality)}\\
  \le  &\sqrt{N} \|\1_{i}\|_2  \|W^t-\frac{1}{N}\1\cdot \1^{\top}\|_2 \\
  \le & \sqrt{N}\lambda_2^t
    \end{split}
\end{equation}
holds for all $i\in [N], k\in [M]$.
\end{proof}

\subsection{Proof for Lemma~\ref{varianceproxy_feducb}}
\begin{customlem}{6}\label{sup_varianceproxy_feducb}
When $n_{i,k}(t)\geq \max\{L,(3M+1)N\}, \forall i\in[N]$, with probability at least $1-p_1,$ the concentration bound for $\tilde \vartheta_{i,k}(t)$ is $
     \mathbb P(|\tilde\vartheta_{i,k}(t) - \mu_{k}|\geq \tilde C_{i,k}(t))<\frac{2}{t^2},$
where $p_1=\frac{2N}{n_{i,k}(t)}+p_0$ and
\begin{equation}\label{eqsup:confidence_fed}
    \tilde C_{i,k}(t)= \alpha_1 + \sqrt{2N\left(\frac{128N \log^2T \cdot \log t \cdot \log n_{i,k}(t)}{n^2_{i,k}(t)\epsilon^2}+\frac{1}{n_{i,k}(t)}\right)\log t}.
\end{equation}
Note $\epsilon$ is the privacy level and parameters $L, p_0, \alpha_1$ are the same as Lemma~\ref{concentration1}.
\end{customlem}
\begin{proof}
Recall $\epsilon' := \epsilon \frac{1}{\lceil \log T \rceil}$.
Up to each time $t$, there are at most $\lceil \log t \rceil $ noise terms added to each $\tilde{X}_{i,k}(t)$.
Let $\gamma_{\tau}\sim {\sf Lap}(1/\epsilon')$ and 
\[
\Gamma:= \sum_{\tau=1}^{\lceil \log t \rceil } \gamma_{\tau}.
\]
The sum of these Laplace noise terms satisfy the following inequality \cite{dwork2006calibrating}:  
\[
\mathbb P(\Gamma \geq \lambda) \leq \exp\left(-\frac{\lambda^2}{8\sum_{\tau} \frac{1}{\epsilon'^2}}\right)
\]
Set $$\lambda:= \frac{4   \log T  \cdot \sqrt{\log t \cdot \log n_{i,k}(t)}}{\epsilon},$$ we have
\begin{align*}
     \exp\left(-\frac{\lambda^2}{8\sum_{\tau} \frac{1}{\epsilon'^2}}\right) \le  &\exp\left(-\frac{\lambda^2}{8 \log t \cdot \frac{\log^2 T}{\epsilon^2}}\right)= \exp\left(-\frac{\frac{16  \log^2 T \cdot \log t \cdot \log n_{i,k}(t)}{\epsilon^2}}{8 \log t \cdot \frac{\log^2 T}{\epsilon^2}}\right) 
=  n_{i,k}^{-2}(t).
\end{align*}
Thus with probability at most ${n_{i,k}^{-2}(t)}$, the noise term added to each empirical mean $\tilde{X}$ is larger than $\frac{\lambda }{n_{i,k}(t)}$. 
Now we want to figure out how the added noise would affect the regret, more specifically, the $\vartheta_{i,k}(t).$ According to \eqref{tilde}, $\vartheta_{i,k}(t)$ is a linear combination of $\tilde{X}_{i,k}(t),$ denote $\Delta \vartheta_{i,k}(t)$ as the total noise added in $\tilde{\vartheta}_{i,k}(t),$  $S_{i,k}(t)$ as the cumulative noise added in $\sum_{\tau=1}^t \tilde{\theta}(\tau).$ Then $\Delta \vartheta_{i,k}(t)$ is a linear combination of $S_{i,k}(t)$ and we have with probability at least $1-p_1,$
{
\begin{align*}
    &|\Delta \vartheta_{i,k}(t)| \\
    =&\bigg|\sum_{j}\bigg([W(t)\cdots W(\tau_{1}+1)-W(t)\cdots W(\tau_{2}+1)]_{i,j}S_{j,k}(\tau_{1})+\cdots\\
    &+\frac{[W(t)\cdots W(\tau_{n_{j,k}(t)-1}+1)-W(t)\cdots W(\tau_{n_{j,k}(t)}+1)]_{i,j}}{n_{j,k}(t)-1} \\
    &\quad \cdot S_{j,k}(\tau_{n_{j,k}(t)-1}) +\frac{[W(t)\cdots W(\tau_{n_{j,k}(t)}+1)]_{i,j}}{n_{j,k}(t)}S_{j,k}(\tau_{n_{j,k}(t)})\bigg)\bigg|\\
    &\leq\sum_{j} |c_{i,k,j}^{(0)}|\cdot4 \log T \cdot \sqrt{ \log t \cdot \log n_{i,k}(t)}/\epsilon\\
    &\leq\frac{8N}{n_{i,k}(t)}\log T  \cdot \sqrt{\log t \cdot \log n_{i,k}(t)}/\epsilon,
\end{align*}}
where $$p_1=2Nn_{i,k}(t)\cdot\frac{1}{n_{i,k}(t)^2}+p_0=\frac{2N}{n_{i,k}(t)}+p_0.$$
Define $\tilde \sigma^2_{i,k}(t)$ as the variance proxy of $\tilde{\vartheta}_{i,k}(t).$
According to Algorithm~\ref{alg:privacy}, each new observation comes with an additional independent Laplacian noise term. Noting the linear combination derived from gossiping procedures will not change the independence, the gossiped noise $\Delta \vartheta_{i,k}(t)$ is independent of each $X_{j,k}(\tau),$ for all  $j\in [N], \tau\in\{0,\ldots,t\}$.
Thus when $n_{i,k}(t)\geq \max\{L,(3M+1)N\},\forall i\in[N]:$
\begin{align*}
    \tilde \sigma^2_{i,k}(t)&=\sigma^2_{i,k}(t)+\sigma^2_{\Delta}    <\frac{64N^2 \log^2T \cdot \log t \cdot \log n_{i,k}(t)}{n^2_{i,k}(t)\epsilon^2}+\frac{N}{2n_{i,k}(t)}.\\
\end{align*}
Following the corresponding procedures as Section~\ref{sec:concentration_bound}, we can prove this lemma.

\end{proof}



\end{document}